\documentclass[reqno]{amsart}
\usepackage{graphicx}
\usepackage{comment}
\usepackage{amsmath,amsthm,amsfonts,amssymb,amscd}
\usepackage[hmarginratio={1:1}]{geometry}
\usepackage{mathtools}
\usepackage{float}
\usepackage{graphicx}
\usepackage{enumitem}
\usepackage{tikz}
\usepackage{thmtools}
\usepackage{thm-restate}
\usepackage[square,sort,comma,numbers]{natbib}
\usepackage{subcaption}
\usepackage{hyperref}
\usepackage{mathrsfs}

\RequirePackage{environ}

\newtoggle{final}
\NewEnviron{add}{
    \iftoggle{final}{\BODY%
    }{%
        {\color{teal}\BODY}%
    }
}
\NewEnviron{del}{%
    \iftoggle{final}{}{%
        {\color{red}\BODY}%
    }
}



 \hypersetup{
     colorlinks=true,
     linkcolor=blue,
     filecolor=blue,
     citecolor = black,      
     urlcolor=cyan,
     }

\newtheorem{remark}{Remark}[section]
\newtheorem{theorem}{Theorem}[section]
\newtheorem{proposition}{Proposition}[section]
\newtheorem{lemma}{Lemma}[section]
\newtheorem{corollary}{Corollary}[lemma]

\theoremstyle{definition}

\newtheorem{example}{Example}[section]
\newtheorem{assumption}{Assumption}
\patchcmd{\subsubsection}{\itshape}{\bfseries}{}{}

\newcommand{\R}{\mathbb{R}}
\newcommand{\N}{\mathbb{N}}

\newcommand{\Z}{\mathbb{Z}}

\newcommand{\E}{\mathbb{E}}
\newcommand{\M}{\mathcal{M}}

\newcommand{\vc}{\vcentcolon}
\newcommand{\eps}{\varepsilon}
\newcommand{\ind}{\mathbf{1}}
\newcommand{\grad}{\mathrm{grad}}

\DeclareMathOperator{\Ima}{Im}

\def\S{{\mathbb S}}

\def\emph{\textit}

\DeclareMathOperator{\Tr}{Tr}
\DeclareMathOperator\supp{supp}
\usepackage[capitalise]{cleveref}

\begin{document}
\title{Manifold learning in metric spaces}
\author{Liane Xu, Amit Singer}

\keywords{Manifold learning, graph Laplacian, Laplacian eigenmaps, diffusion maps, Wasserstein space.}

\begin{abstract}
    Laplacian-based methods are popular for the dimensionality reduction of data lying in $\R^N$. Several theoretical results for these algorithms depend on the fact that the Euclidean distance locally approximates the geodesic distance on the underlying submanifold which the data are assumed to lie on.
    However, for some applications, other metrics, such as the Wasserstein distance, may provide a more appropriate notion of distance than the Euclidean distance. 
    We provide a framework that generalizes the problem of manifold learning to metric spaces and study when a metric satisfies sufficient conditions for the pointwise convergence of the graph Laplacian.
\end{abstract}

\maketitle

\section{Introduction}

In the manifold learning problem in $\R^N$, we assume that data points $x_1, ..., x_n \in \R^N$ are drawn i.i.d. from some probability distribution on an unknown compact, oriented submanifold $\mathcal{M}\subset \R^N,$ and we would like to recover (properties of) $\mathcal{M}$. We can formulate an analogous problem of manifold learning in a metric space $(X, d)$ as follows: let $\mathcal{M}$ be a smooth, compact, oriented manifold, and let $\iota \colon \mathcal{M} \to X$ be an embedding. 
Given $x_1, ..., x_n$ drawn i.i.d. from some probability distribution on $\mathcal{M}$, what can we recover about $\mathcal{M}$?

Several popular manifold learning algorithms for data points in $\R^N$ — for example, Isomap \cite{tenenbaum2000global}, Laplacian eigenmaps \cite{belkin2001laplacian, belkin2003laplacian} and diffusion maps \cite{coifman2006diffusion} — start by constructing a graph which depends on the pairwise Euclidean distances between the data points. 
To motivate why manifold learning in another metric space may be beneficial, consider the following example. Fix some $r > 0$.
Let $\iota \colon \S^1 \to L^2(\R^2)$ be given by
\begin{equation}
\label{eq:iota}
\iota(\theta) \vc = \frac{1}{\pi r^2} \ind_{B_r((\cos \theta, \sin \theta))} 
\end{equation}
for $\theta \in [0, 2\pi)$, where $B_r(x)$ denotes an open ball of radius $r$ centered at $x$ and $\ind_A$ denotes the indicator function for a set $A$. We can think of $\iota(\theta)$ (after some discretization) as an image of a ball of radius $r$ centered at $(\cos \theta, \sin \theta)$. For a fine enough discretization, the Euclidean norm approximates the $L^2(\R^2)$ norm, up to a scaling factor. 

We can explicitly compute the pairwise $L^2(\R^2)$ distances for this example.
In fact, for $\theta, \varphi \in [0, 2\pi)$ and $r \in (0, 1]$, the angle between $\theta$ and $\varphi$ is strictly greater than $2 \arcsin {r}$ if and only if $|(\cos \theta, \sin \theta) - (\cos \varphi, \sin \varphi)| > 2r$, and when $|(\cos \theta, \sin \theta) - (\cos \varphi, \sin \varphi)| > 2r$, we have 
\[ \|\iota(\theta) - \iota(\varphi)\|_{L^2(\R^2)} = \frac{\sqrt{2}}{r \sqrt{\pi}}\]
since the supports of $\iota(\theta)$ and $\iota(\varphi)$ are disjoint.
Now suppose $\theta_1, ..., \theta_n$ are drawn i.i.d. from the uniform distribution on $[0, 2\pi)$.
If $|(\cos \theta_j, \sin \theta_j) - (\cos \theta_1, \sin \theta_1)| > 2r$ for all $j > 1$, then 
\[\|\iota(\theta_j) - \iota(\theta_1)\|_{L^2(\R^2)} = \|\iota(\theta_i) -\iota(\theta_1)\|_{L^2(\R^2)}\]
for all $i, j > 1$, so we have \textit{no information} on the relative location of $\theta_1$ compared to the other points.
This happens with probability
\[ \mathbb{P}\left( |(\cos \theta_j, \sin \theta_j) - (\cos \theta_1, \sin \theta_1)| > 2 r \text{ for all } j > 1 \right) = \left( 1- \frac{2 \arcsin r}{\pi} \right)^{n-1}. \]
Although this decays exponentially as $n \to \infty$, in practice, we may only have limited samples, so using the $L^2(\R^2)$ norm in this case would detrimentally affect the output of our algorithm (for Isomap, Laplacian eigenmaps, diffusion maps — \textit{any} manifold learning algorithm that uses pairwise distances between data points), especially for small sample sizes and small $r$. 

The issue here is that for small $r > 0$, the support of $\iota(\theta)$ is small for each $\theta \in [0, 2\pi)$.
To alleviate this, we could consider using a different metric. For this particular $\iota$, the Wasserstein-2 distance between $\iota(\theta)$ and $\iota(\varphi)$ for any $\theta, \varphi \in [0, 2\pi)$ can also be computed explicitly: for $\theta, \varphi \in [0, 2\pi),$
\[ W_2^2(\iota(\theta), \iota(\varphi)) = |(\cos \theta, \sin \theta) - (\cos \varphi, \sin \varphi)|^2, \]
where, by a slight abuse of notation, we use $\iota(\theta)$ to denote the probability measure on $\R^2$ with density $\iota(\theta).$ For this example, one might think that the Wasserstein-2 distance provides a ``better" notion of distance than the $L^2(\R^2)$ norm. 

In this case, the use of the Wasserstein-2 metric with the embedding $\iota$ turns out to be equivalent to considering the typical embedding $\theta \mapsto (\cos \theta, \sin \theta)$ of $\S^1$ in $\R^2$.
However, more generally, we still lack a theoretical understanding of the use of non-Euclidean metrics in graph Laplacian-based manifold learning algorithms. 
Much work has been done in the Euclidean case (e.g., \cite{belkin2001laplacian, belkin2003laplacian, belkin2005towards, belkin2008towards, coifman2005geometric, coifman2006diffusion, lafon2004diffusion, singer2006graph, hein2005graphs, berry2016variable, cheng2022eigen, cheng2022knnselftune, belkin2006convergence} and references therein), but less is known for other metrics 
\cite{kileel2021manifold, trillos2024fermat}.
To the best of our knowledge, the convergence of graph Laplacians with an arbitrary metric has not been studied before. Moreover, a theoretical understanding of why certain metrics produce better results for a given problem is also missing.

In this paper, we provide a partial answer. Henceforth, for convenience, we treat $d$ as a metric on $\M$ and write $d(x, y)$ instead of $d(\iota(x), \iota(y)).$ Our first main contribution is proving that if $\M$ is equipped with a Riemannian metric $g$, $d_g$ is the associated geodesic distance and $\exp$ is the associated exponential map, the following two assumptions suffice for the pointwise convergence of graph Laplacians:

\begin{restatable}[Uniform first order approximation]{assumption}{firstorderassump}
    \label{assump:firstorder}
    There exists a function $R \colon [0, \infty) \to [0, \infty]$ such that $R(u)/u \to 0$ as $u \to 0^{+}$ and for any length-minimizing, unit-speed geodesic segment $\gamma \colon [0, t] \to \M$,
    \[ |d(\gamma(t), \gamma(0)) - t| \leq R(t). \] 
    Here length-minimizing means that $d_g(\gamma(t), \gamma(0)) = t$.

\end{restatable}

\begin{restatable}{assumption}{assumptwo}
    \label{assump:thirdorder}
    There exists $K, \eps_0 > 0$ such that $d_g^2(x, y) - d^2(x, y) < Kd^4_g(x, y)$ for all $x, y\in \M$ such that $d_g(x, y) < \eps_0$.
\end{restatable}

The benefits of this approach are twofold. First, Assumption \ref{assump:thirdorder} gives us a way to quantify how well $d$ approximates $d_g$. Secondly, in contrast to \cite{kileel2021manifold}, these assumptions allow us to follow the proof of the pointwise convergence of graph Laplacians in the Euclidean case \cite{belkin2005towards, belkin2008towards} almost word-for-word; even more, this proof gives us upper bounds on the bias error in terms of $K$ and $\eps_0$. This is a first step towards understanding an observation that Kileel et al. made with their numerical experiments \cite[Section 5]{kileel2021manifold}, which can be summarized roughly as follows: for two metrics $d, \tilde{d}$ on a Riemannian manifold $(\M, g)$, if $y \mapsto d(x, y)$ and $y \mapsto \tilde{d}(x, y)$ both locally approximate $y \mapsto d_g(x, y)$, but $y \mapsto d(x, y)$ does so on a larger ball around $x$ for all $x \in \M$, then the Laplacian eigenmaps algorithm tends to perform better with $d$ than $\tilde{d}.$

Our second main contribution is observing that if $d^2$ is sufficiently regular around the diagonal of $\M \times \M$ and
\begin{equation}
\label{eq:induced-metric}
g_x \vc= \frac{1}{2}\mathrm{Hess}_x\left(d^2(\cdot, x) \right) \ \ \forall x \in \M 
\end{equation}
is nondegenerate for all $x \in \M$, then (\ref{eq:induced-metric}) is the unique $C^2$ Riemannian metric on $\M$ for which Assumptions \ref{assump:firstorder} and \ref{assump:thirdorder} are satisfied. Here $\mathrm{Hess}_x$ denotes the Hessian evaluated at $x$, and $d^2(\cdot, x)$ denotes the map $y \mapsto d^2(y, x)$. (Since $d^2(\cdot, x)$ has a critical point at $x$, the Hessian can be taken with respect to any smooth Riemannian metric on $\M$.)
This allows us to more easily check when the graph Laplacians converge pointwise and to identify the limiting operator.

The rest of the paper is organized as follows. In Section \ref{sec:background}, we provide some background on manifold learning and optimal transport and describe related work.
In Section \ref{sec:glconv}, we give conditions on $d$ which are sufficient for the pointwise convergence of graph Laplacians. 
In \cref{sec:examples}, we provide several examples of when the Wasserstein-2 and other distances satisfy these conditions, and in \cref{sec:numexp}, we perform numerical experiments which support our results from Section \ref{sec:glconv}.

\subsection{Notation and assumptions}
\label{subsec:notation-and-assumptions}
\label{subsec:notation}
All of our results in Section \ref{sec:glconv} hold when $\M$ is a smooth, compact, oriented manifold. It is possible to assume looser conditions: for example, following more or less the same proof, \cref{thm:glconv_unif} also holds when $\M$ is just a subset of finite volume of a smooth, oriented Riemannian manifold $(\widetilde{\M}, g)$ and $x$ belongs to the interior of $\M$. However, for simplicity, we take $\M$ to be a smooth, oriented manifold (without boundary). We will provide an example of a non-compact manifold in \cref{subsec:weightedl1}, so we will not require compactness. 

Unless otherwise specified, $\iota \colon \M \to (X, d)$ is a topological embedding of $\M$ into a metric space $(X, d)$, i.e., $\iota$ is a homeomorphism between $\M$ and $\iota(\M)$, where $\M$ is equipped with the topology from its manifold structure and $\iota(\M)$ is equipped with the subspace topology from $(X, d)$.
If $\M$ is compact, then any injective continuous map $\iota \colon \M \to X$ is a topological embedding. However, if $\M$ is not compact, this is not true. 
 That $\iota$ is a topological embedding is required for \cref{lemma:betapos}, which is used for our results in \cref{subsec:suff}.

When there is no ambiguity, we will often treat $d$ as a metric on $\M$. Requiring $\iota$ to be a topological embedding is equivalent to requiring that the topology on $\M$ generated by $d$ is the same as the topology from the manifold structure of $\M$ (see also \cref{rmk:compare-w-arias}).

We warn the reader that we will make different assumptions for $\M$ in Section \ref{subsec:suff} and Section \ref{subsec:suffdiff}:
\begin{itemize}[label = {--}]
    \item In Section \ref{subsec:suff}, we equip $\M$ with a Riemannian metric $g$ and its Levi-Civita connection $\nabla$. We assume that $\M$ has finite volume with respect to $g$ and that $g$ is $C^{k+1}$ for some $k \in \N$, $k \geq 3$. We denote the geodesic distance on $(\M, g)$ by $d_g$. The volume form is denoted by $dV_g$, and the volume of a Borel measurable set $U \subset \M$ is denoted by $\mathrm{vol}_g(U).$ The gradient is denoted by $\grad_g$.
   We use $\mathrm{Hess}$ to denote the Hessian on $(\M, g)$, which we recall is defined by
    \begin{equation}
    \label{eq:defhess-1}
    \mathrm{Hess} f(X, Y) \vc = X(Yf) - \left( \nabla_X Y \right)f 
    \end{equation}
    for $f \in C^2(\M)$ and smooth vector fields $X, Y$ on $\M$.
    We denote by $\Delta_g$ the Laplace-Beltrami operator (also called the Laplacian) on $(\M, g)$: 
    \begin{equation}
    \label{eq:deflap}
        \Delta_g f \vc = -\Tr \left( \mathrm{Hess} f \right).
    \end{equation} Since $g$ is $C^{k+1}$, the exponential map $\exp$ is $C^k$ around the zero section of the tangent bundle $T\M$ (see \cref{subsec:diff-geo} for more details).
We denote the injectivity radius at $x \in \M$ by
\begin{equation}
\mathrm{inj}_{(\M, g)}(x) \vc = \sup \{ \delta > 0 \colon \exp_x \text{ is a } C^k\text{-diffeomorphism on } B_\delta(0) \subset T_x\M \}.
\end{equation}
    \item In Section \ref{subsec:suffdiff}, a priori, we do not equip $\M$ with a Riemannian metric. Rather, we show that if $d^2$ is sufficiently regular around the diagonal of $\M \times \M$ and
    \begin{equation}
    \label{eq:hess-metric-nota}
    g_x = \frac{1}{2}\mathrm{Hess}_x\left(d^2(\cdot, x) \right) 
    \end{equation}
    is nondegenerate for all $x \in \M$, then (\ref{eq:hess-metric-nota})
    is the unique $C^2$ Riemannian metric on $\M$ such that the results of \cref{subsec:suff} can hold. Here $d^2(\cdot, x)$ denotes the map $y \mapsto d^2(y, x)$. Since $d^2(\cdot, x)$ has a critical point at $x$, the Hessian can be taken with respect to any smooth Riemannian metric on $\M$; in other words, by (\ref{eq:defhess-1}), if $x$ is a critical point of $f \in C^2(\M)$,
    \[ \mathrm{Hess}_x f(X, Y) = X(Yf)(x) \]
    for smooth vector fields $X, Y$ on $\M$, independent of any Riemannian metric on $\M$.
    We denote the diagonal of $\M \times \M$ by
    \[ D_\M \vc = \{ (x, x) \colon x \in \M \}. \]
\end{itemize}

We provide a brief overview of some Riemannian geometry in \cref{subsec:diff-geo}.
Throughout the paper, unless otherwise stated, smooth refers to $C^\infty$, $m$ is the dimension of $\M$, $n$ is the sample size, and
for $k \in \N$, $[k]$ denotes the set $\{ 1, \ldots, k \}$.

\section{Background}
\label{sec:background}

\subsection{Manifold learning}
We first briefly describe results from manifold learning in the Euclidean setting, with a focus on methods based on the graph Laplacian \cite{belkin2001laplacian, belkin2003laplacian, belkin2005towards, belkin2008towards, coifman2005geometric, coifman2006diffusion, lafon2004diffusion}. 
Let $\M \subset \R^N$ be a $m$-dimensional smooth, compact, oriented Riemannian submanifold of $\R^N$, and suppose we have data points $x_1,..., x_n$ drawn i.i.d. uniformly from $\M$. In various applications, $N$ may be large, and our goal is to perform dimensionality reduction, i.e., map each $x_i$ into $\R^l$ for $l \ll N.$

One popular method to do this is Laplacian eigenmaps, introduced by Belkin and Niyogi \cite{belkin2001laplacian, belkin2003laplacian}.
For each $\eps > 0$, define $k_\eps \colon \R^N \times \R^N \to \R$ by
\begin{equation}
\label{eq:kernel}
k_\eps(x, y) \vc =  e^{\textstyle -\frac{d(x, y)^2}{2\eps}}, 
\end{equation}
where $d(x, y) = \|x-y\|_{\R^N},$ the typical Euclidean distance.
Now construct a graph whose vertices are $\{x_i\}_{i \in [n]}$, and set the edge weight matrix $W \in \R^{n \times n}$ and degree matrix $D \in \R^{n \times n}$ via
\[ W_{ij} = k_\eps(x_i, x_j) \]
\[ D_{ij} = \begin{cases}
    \sum_{k} W_{ik} & i = j \\
    0 & i \neq j
\end{cases}
\]
for $i, j  \in [n]$.\footnote{In \cite{belkin2001laplacian, belkin2003laplacian}, a $\delta$-neighborhood rule or $k$NN rule is also used when constructing the adjacency graph, i.e., $W_{ij} = 0$ if $d(x_i, x_j)^2 \geq \delta$ ($\delta$-neighborhood rule) or if $x_j$ is not one of the $k$ nearest neighbors of $x_i$ ($k$NN rule). Following \cite{belkin2005towards, belkin2008towards}, we do not do this; this simplifies our analysis, and we leave extensions to the $\delta$-neighborhood or $k$NN case to future work.} The graph Laplacian $L^{(\eps, n)}$ is then defined by
\begin{equation}
\label{eq:gl}
L^{(\eps, n)} \vc = D - W 
\end{equation}
and the (random walk) normalized graph Laplacian $\mathcal{L}^{(\eps, n)}$ is given by
\begin{equation}
\label{eq:ngl}
\mathcal{L}^{(\eps, n)} \vc = I - D^{-1}W. 
\end{equation}
More explicitly, for $f \colon \M \to \R$,
\begin{equation}
L^{(\eps, n)} \left[ f(x_i) \right]_i = \left[ \sum_{j} k_\eps(x_i, x_j)(f(x_i) - f(x_j))  \right]_i. 
\end{equation}
Hence we can define the discrete Laplacian $L^{(\eps, n)}f$ of $f$ by
\begin{equation}
L^{(\eps, n)}f(x) \vc = \sum_j k_\eps(x, x_j)(f(x) - f(x_j)) 
\end{equation}
for $x \in \M$, and similarly,
\begin{equation}
\mathcal{L}^{(\eps, n)}f (x) \vc = f(x) -\frac{\sum_j k_\eps(x, x_j)f(x_j)}{\sum_j k_\eps(x, x_j)}.
\end{equation}
Since $D^{-1}W$ is a Markov transition matrix, it has an eigenvalue of 1 with eigenvector $[1 \cdots 1]^T$, and any other eigenvalue $\lambda$ of $D^{-1}W$ satisfies $\lambda \leq 1$. It follows that $\mathcal{L}^{(\eps, n)}$ has eigenvalues $\lambda_0, ..., \lambda_{n-1}$ which satisfy
\[ 0 = \lambda_0 \leq \lambda_1 \leq \cdots \leq \lambda_{n-1}, \]
with $v^{(0)} = [1 \cdots 1]^T$ the eigenvector for $\lambda_0$. 
Laplacian eigenmaps then performs dimensionality reduction by computing the eigenvectors $v^{(1)}, ..., v^{(l)}$ associated with $\lambda_1, ..., \lambda_{l}$ and mapping each $x_i$ into $\R^l$ by
\begin{equation} 
\label{eq:lapeigenmap}
x_i \mapsto \begin{bmatrix}
    v^{(1)}_i \\
    \vdots \\
    v^{(l)}_i
\end{bmatrix}.
\end{equation}
The diffusion maps algorithm \cite{coifman2005geometric, coifman2006diffusion, lafon2004diffusion, nadler2006diffusion} is similar in that it considers the spectral decomposition of the Markov transition matrix $D^{-1}W$ (or other normalizations of the weight matrix).

The theoretical motivation for these graph Laplacian-based methods lies in the fact that the graph Laplacian approximates (in several senses, to be made precise later) the Laplace-Beltrami operator $\Delta_g$ on $\M$. Notably, on a compact, connected, oriented smooth Riemannian manifold $(\M, g)$, there exist smooth eigenfunctions $\{\varphi^{(k)}\}_{k \in \Z_{\geq 0}}$ of $\Delta_g$ such that $\{\varphi^{(k)}\}_{k \in \Z_{\geq 0}}$ is an orthonormal basis of $L^2(\M)$, $\varphi^{(0)}$ is constant on $\M$ and the associated eigenvalues $\{\lambda_k\}_{k \in \Z_{\geq 0}}$ are ordered:
\[ 0 = \lambda_0 < \lambda_1 \leq \cdots \leq \lambda_k \leq \lambda_{k+1} \leq \cdots . \]
See, e.g., \cite{rosenberg1997laplacian} for a proof and further discussion of the importance of the Laplace-Beltrami operator on Riemannian manifolds.
As an example, for $\M = \S^1$ with the standard round metric $g$, we can take $\{\varphi^{(k)}\}_{k \in \Z_{\geq 0}}$ to be the Fourier basis (appropriately normalized), ordered by increasing frequency.
As a corollary,
\begin{equation}
\label{eq:eigenfuncs_all}
x \mapsto \left[ \varphi^{(k)}(x) \right]_{k \in \N} 
\end{equation}
is an injective map on $\M$ \cite{berard1994embedding}, and in fact, it suffices to take finitely many eigenfunctions: there exists $\tilde{l} < \infty$ such that
\begin{equation}
\label{eq:eigenfuncs}
x \mapsto \left[ \varphi^{(k)}(x) \right]_{k \in \left[\tilde{l}\right]} 
\end{equation}
is a smooth embedding of $\M$ into $\R^{\tilde{l}}$ \cite{abdalla2012embedding, bates2014embedding, berard1985volume}.
(\ref{eq:lapeigenmap}) can then be viewed as a discrete approximation of (\ref{eq:eigenfuncs}).

More precisely, if $(\M, g)$ is a smooth, compact, oriented $m$-dimensional Riemannian submanifold of $\R^N$, $x_1, x_2, ..., x_n$ are drawn i.i.d. from the uniform distribution on $\M$, $\eps_n = 2n^{-\frac{1}{m+2+\alpha}}$ for some $\alpha > 0$, $x \in \M$ and $f \in C^\infty(\M)$, then Belkin and Niyogi showed that
\begin{equation}
\label{eq:glconv_euc_unif}
\frac{2}{\eps_n (2\pi \eps_n)^{m/2}n} L^{(\eps_n, n)} f(x) \to \frac{\Delta_g f(x)}{\mathrm{vol}_g(\M)}  
\end{equation}
in probability as $n \to \infty$ \cite{belkin2005towards, belkin2008towards}. A similar result can be shown for normalized graph Laplacians; the convergence rate was studied in \cite{hein2005graphs, singer2006graph}.
More generally, if $x_1, x_2, ..., x_n$ are drawn from a probability distribution on $\M$ with density $P \in C^3(\M)$, the limit in (\ref{eq:glconv_euc_unif}) is instead $P \Delta_g f - 2\langle \grad_g f, \grad_g P\rangle$
\cite{belkin2008towards, coifman2005geometric, coifman2006diffusion, lafon2004diffusion, hein2005graphs, nadler2006diffusion}. Belkin and Niyogi also established spectral convergence when $x_1, x_2, \ldots, x_n$ are drawn from the uniform distribution on $\M$ \cite{belkin2006convergence}, and a number of recent works have studied the convergence rate for the spectral convergence of (variants of) the graph Laplacian \cite{garcia2020error, cheng2022eigen, calder2022improved, calder2022lipschitz, dunson2021spectral, wahl2024kernel, tan2024improved, trillos2025minimax, cheng2024bi}. 
We note that several of the aforementioned works (e.g., \cite{hein2005graphs, coifman2005geometric, coifman2006diffusion, garcia2020error, calder2022improved, cheng2022eigen, calder2022lipschitz, cheng2024bi}) and others  (e.g., \cite{ting2010analysis, cheng2022knnselftune, berry2016variable})  establish convergence results for different variants of graph Laplacians; for example, for kernels other than the Gaussian kernel (\ref{eq:kernel}), for normalizations other than the random walk normalization (\ref{eq:ngl}), for graphs constructed with the $\delta$-neighborhood rule or $k$NN rule, etc. For simplicity, we only consider the graph Laplacian and its random walk normalization as we have introduced them in (\ref{eq:gl}) and (\ref{eq:ngl}) respectively, and leave extensions to other variants of the graph Laplacian to future work.

\subsection{Optimal transport}

As noted in the introduction, one of the metrics that we are interested in is the Wasserstein distance. Let $\mathcal{P}(X)$ denote the set of Borel probability measures on a metric space $X$, and for $p \geq 1$, let $\mathcal{P}_p(\R^N)$ denote the set of Borel probability measures with finite $p$th moments on $\R^N$. The Wasserstein-$p$ distance $W_p(\mu, \nu)$ between $\mu, \nu \in \mathcal{P}_p(\R^N)$ is given by
\begin{equation}
    W_p(\mu, \nu) \vc = \inf_{\gamma \in \Gamma_{\mu, \nu}} \left( \int_{\R^N} \|x-y\|^p d\gamma(x, y) \right)^{1/p},
\end{equation}
where $\Gamma_{\mu, \nu}$ is the set of couplings between $\mu$ and $\nu$, i.e., the set of probability measures $\gamma \in \mathcal{P}(\R^N \times \R^N)$ whose first marginal is $\mu$ and second marginal is $\nu$.

We will mainly be concerned with the case $p = 2$ because $(\mathcal{P}_2(\R^N), W_2)$ has a formal Riemannian structure. If we can embed a manifold $\M$ into $(\mathcal{P}_2(\R^N), W_2)$ and equip it with this formal Riemannian structure, then we can understand $\M$ (at least formally) as a \textit{Riemannian} submanifold of $(\mathcal{P}_2(\R^N), W_2)$. If this turns out to be an actual Riemannian metric $g$ on $\M$, we can then ask when the graph Laplacians converge to the Laplace-Beltrami operator $\Delta_g$ induced by $g$, in direct analogy with the manifold learning problem in Euclidean space.

In what follows, we briefly review the formal Riemannian structure of $(\mathcal{P}_2(\R^N), W_2)$.
For more details, we refer the reader to \cite{ambrosio2008gradient}.
For more on optimal transport more broadly, we refer the reader to one of several books on optimal transport, e.g., \cite{villani2009optimaloldandnew, villani2021topics, santambrogio}.

\subsubsection{Riemannian structure of Wasserstein space}
\label{subsubsec:wass_riem}
Let $(X, d)$ be a complete metric space and $I = (a, b)$ for some $a < b$. 
Recall that a curve
$\gamma \colon I \to X$ is absolutely continuous
if there exists $f \in L^1(I)$ such that 
\begin{equation}
    \label{eq:abscont}
d(\gamma(s), \gamma(t)) \leq \int_s^t f(x) dx 
\end{equation}
for all $s, t \in I$ such that $s \leq t$ \cite[Definition 1.1.1]{ambrosio2008gradient}. Moreover, if $\gamma$ is absolutely continuous, the 
metric derivative, defined as 
\[ |\gamma'|(t) \vc = \lim_{\eps \to 0} \frac{d(\gamma(t+\eps), \gamma(t))}{|\eps|}, \]
exists for a.e. $t \in I$ and $|\gamma'| \in L^1(I)$ \cite[Theorem 1.1.2]{ambrosio2008gradient}. Here and henceforth a.e. denotes almost everywhere with respect to the Lebesgue measure.

In the case when $(X, d) = \left(\mathcal{P}_2\left(\R^N\right), W_2\right)$, it can be shown \cite[Theorem 8.3.1]{ambrosio2008gradient} that
any absolutely continuous curve
\begin{align*}
\mu \colon I &\to \mathcal{P}_2(\R^N) \\
t &\mapsto \mu_t
\end{align*}
admits a Borel measurable time-dependent vector field 
\begin{align*}
v \colon \R^N \times I &\to \R^N \\
(x, t) &\mapsto v_t(x)
\end{align*}
such that
\begin{itemize}[label = {--}]
\item the continuity equation
\[ \frac{\partial \mu}{\partial t} + \mathrm{div}\left( v_t \mu_t \right) = 0 \] 
is satisfied in the distributional sense:
\begin{equation}
    \label{eq:conteq}
\int_I \int_{\R^N} \frac{\partial \varphi}{\partial t} + \langle v_t, \nabla \varphi \rangle d\mu_t dt = 0 
\end{equation}
for every $\varphi \in C_c^\infty(\R^N \times I),$ where $\nabla$ is the gradient on $\R^N$;
\item $v_t \in L^2(\mu_t)$ and $\|v_t\|_{L^2(\mu_t)} = |\mu'|(t)$ 
for a.e. $t \in I$;
\item $v_t \in \overline{\{ \nabla \varphi \colon \varphi \in C_c^\infty(\R^N) \}}^{L^2(\mu_t)}$ for a.e. $t \in I$, where $\overline{V}^{L^2(\mu_t)}$ denotes the closure of $V \subset L^2(\mu_t)$ in $L^2(\mu_t).$
\end{itemize}

On the other hand, suppose a Borel measurable vector field $v \colon \R^N \times I \to \R^N$ 
satisfies the continuity equation for an absolutely continuous curve $\mu \colon I \to \mathcal{P}_2(\R^N)$ in the distributional sense and $t \mapsto \|v_t\|_{L^2(\mu_t)}$ belongs to $L^1(I)$. Then
\[ |\mu'|(t) \leq \|v_t\|_{L^2(\mu_t)} \]
for a.e. $t \in I$ \cite[Theorem 8.3.1]{ambrosio2008gradient}.
In addition, if both $v \colon \R^N \times I \to \R^N$ and $\widetilde{v} \colon \R^N \times I \to \R^N$ are Borel measurable, satisfy the continuity equation and $v_t, \widetilde{v}_t \in \overline{\{ \nabla \varphi \colon \varphi \in C_c^\infty(\R^N) \}}^{L^2(\mu_t)}$ for a.e. $t \in I$, then $v_t = \widetilde{v}_t$ for a.e. $t \in I.$ 
Indeed, since $v$ and $\widetilde{v}$ both satisfy the continuity equation, their difference is divergence-free:
\[ \mathrm{div} \left( (v_t - \tilde{v}_t) \mu_t \right) = 0 \]
for a.e. $t \in I$.
By definition, for any $\nu \in \mathcal{P}_2(\R^N)$ and vector field $w \in L^2(\nu)$,
$\mathrm{div}(w\nu) = 0$ means that
\[ \langle \nabla \varphi, w \rangle_{L^2(\nu)} = 0 \]
for all $\varphi \in C_c^\infty(\R^N)$, so the orthogonal complement of $\{ w \colon \mathrm{div}(w \nu) = 0 \}$ in $L^2(\nu)$ is 
\[\{ w \colon \mathrm{div}(w \nu) = 0 \}^{\perp} = \overline{\{ \nabla \varphi \colon \varphi \in C_c^\infty(\R^N) \}}^{L^2(\nu)}.\] 
Hence $v_t = \widetilde{v}_t$ for a.e. $t \in I$. Altogether, this motivates the definition for the tangent space at $\nu \in \mathcal{P}_2(\R^N)$:
\[ T_\nu \mathcal{P}_2(\R^N) \vc = \overline{\{ \nabla \varphi \colon \varphi \in C_c^\infty(\R^N) \}}^{L^2(\nu)}. \] 
The formal Riemannian structure of $\mathcal{P}_2(\R^N)$ is then obtained by taking the $L^2(\nu)$ inner product on $T_\nu \mathcal{P}_2(\R^N)$ for each $\nu \in \mathcal{P}_2(\R^N)$.

In short, for any absolutely continuous curve $\mu \colon I \to \mathcal{P}_2(\R^N)$, there exists a time-dependent vector field $\{v_t\}_{t \in I}$
satisfying the continuity equation for $\mu$ such that $v_t \in T_{\mu_t}\mathcal{P}_2(\R^N)$ and $\|v_t\|_{L^2(\mu_t)} = |\mu'|(t)$ for a.e. $t \in I$, and such a vector field is uniquely determined for a.e. $t \in I$.

\subsection{Related work}

The idea of using other metrics in graph Laplacian-based methods for dimensionality reduction is not new.
Zelesko et al. proposed using a wavelet approximation of the Wasserstein-1 metric \cite{shirdhonkar2008approximate} in the diffusion maps algorithm \cite{zelesko2020earthmover} (see also \cite{kileel2021manifold}) to study the conformation space of a molecule. Lu et al. used the diffusion maps algorithm with (an approximation of) the Wasserstein-2 metric to identify conserved quantities in a dynamical system \cite{lu2023discovering}. In addition, Carter et al. proposed estimating the geodesic distance on a statistical manifold with the Fisher information as a Riemannian metric and using it in multidimensional scaling (\`a la Isomap) or Laplacian eigenmaps \cite{carter2009fine}. Mishne et al. defined a metric based on a tree constructed from the data and used this metric in the diffusion maps algorithm to analyze neuronal activity \cite{mishne2016hierarchical}. However, for this paper, we assume that the metric is independent of the data, so our results do not apply to the metrics used in \cite{carter2009fine, mishne2016hierarchical}.

Dimensionality reduction methods other than graph Laplacian-based methods have been studied with non-Euclidean metrics (and even ways of measuring dissimilarity which are not metrics) as well. In fact, the MNIST example in the original Isomap paper does not use the Euclidean distance between the digits as the input to the Isomap algorithm \cite[Fig. 1B]{tenenbaum2000global}. The Isomap algorithm itself can be viewed as multidimensional scaling with approximations of the manifold's geodesic distance \cite{tenenbaum2000global}. Several works have also looked at extending various dimensionality reduction methods
to a space of probability measures equipped with the Wasserstein distance (or approximations thereof) \cite{hamm2023wassmap, bigot2017geodesic, seguy2015principal, collas2023entropic, flamary2018wasserstein}. Negrini and Nurbekyan recently proposed using a type of transportation distance called the no-collision distance in multidimensional scaling \cite{negrini2024applications}. Another line of recent work has studied multidimensional scaling on metric measure spaces \cite{adams2020multidimensional, kroshnin2022infinite, lim2024classical}; see also \cite{arias2025embedding}.

The two works closest to ours are \cite{kileel2021manifold} and \cite{trillos2024fermat}.
Kileel et al. \cite{kileel2021manifold} establish the pointwise convergence of the graph Laplacian to a Laplacian-like operator when $\M$ is a smooth submanifold of $\R^N$ with an arbitrary norm.
In contrast, we do not require $\M$ to be a submanifold of $\R^N$, and we consider metrics, not just norms. Moreover, the coefficients of the limiting operator in \cite{kileel2021manifold} are given in terms of extrinsic quantities such as the second fundamental form, whereas the limiting operators in \cref{thm:glconv_unif} and \cref{thm:glconv_nonunif} are intrinsic. In \cref{subsec:weightedl1}, we show how we can apply our approach to \cite[Example 3.7]{kileel2021manifold}. However, the assumptions we make for \cref{thm:c4} may not be satisfied by an arbitrary norm and smooth submanifold of $\R^N$. Additionally, Garc\'ia Trillos et al. prove the spectral convergence of graph Laplacians when the graph Laplacians are constructed with a metric from a specific family of data-dependent distances known as Fermat distances \cite{trillos2024fermat}. To the best of our knowledge, the pointwise convergence of graph Laplacians constructed with Fermat distances is still unknown, but as we mentioned above, our results do not apply to metrics which depend on the data.

We also point out two other theoretical results on manifold learning — though not graph Laplacian-based methods — with non-Euclidean metrics. Hamm et al. \cite{hamm2024manifoldlearningwassersteinspace} focus on the space of absolutely continuous probability measures $\mathcal{P}_{ac}(\Omega)$ on a compact, convex subset $\Omega \subset \R^N$ equipped with the Wasserstein-2 distance $W_2$. They prove the convergence of graph metric measure spaces in the $\infty$-Gromov-Wasserstein distance when there is a bi-Lipschitz embedding (with some additional regularity conditions) of a compact, connected smooth Riemannian manifold $(\M, g)$ into $(\mathcal{P}_{ac}(\Omega), W_2)$. Arias-Castro and Qiao \cite{arias2025embedding} consider the following model: suppose $\Theta \subset \R^m$ is open and connected, $X \vc= \{ f_\theta \colon \theta \in \Theta\}$ is a family of probability densities with respect to a fixed measure $\mu$ on $\R^N$, and $d$ is a metric on $X$. Here $\Theta$ can be viewed as the space of parameters. They then study the behavior of Isomap given a finite set $\theta_1, ..., \theta_n \in \Theta$, a finite set of samples $S_{\theta_i}$ drawn from each $f_{\theta_i}d\mu$ and approximations of $d\left( f_{\theta_i}, f_{\theta_j} \right)$ constructed from $S_{\theta_i}$ and $S_{\theta_j}$. To analyze Isomap in this setting, if $d^2$ is $C^2$ on $\Theta$ and (\ref{eq:induced-metric}) is nondegenerate for all $x \in \Theta$, they equip $\Theta$ with the Riemannian metric (\ref{eq:induced-metric}). 

Although not directly related to manifold learning, there has been some work \cite{chen2020optimal, li2023wasserstein} on when $d$ is the Wasserstein-2 distance (see also \cite{li2018natural} for a graph variant) and a parameter space $\Theta \subset \R^m$ is equipped with a Riemannian metric that agrees with (\ref{eq:induced-metric}) when appropriate conditions are satisfied; this is motivated in part by the multitude of results related to the Fisher information. Indeed, if $d$ is the Hellinger distance, or if $d^2$ in (\ref{eq:induced-metric}) is replaced by the KL divergence, under some assumptions, (\ref{eq:induced-metric}) becomes a scalar multiple of the Fisher information \cite{chewi2024statisticaloptimaltransport, kullback1997information}. 
The idea of defining a Riemannian metric on a parametrized family of probability measures via the Hessian of a divergence dates back to Rao \cite{rao1987differential}, and the use of the Fisher information as a Riemannian metric is well-studied in information geometry (see, e.g., \cite{ay2017information}). To the best of our knowledge, the use of (\ref{eq:induced-metric}) to understand the pointwise convergence of graph Laplacians is new.

\section{Pointwise convergence of the graph Laplacian}
\label{sec:glconv}
Let us start from the Euclidean case.
Consider a smooth oriented Riemannian submanifold $(\M, g)$ of $\R^N$ of finite volume. Let $x \in \M$ and $f \in C^3(\M) \cap L^\infty(\M)$. Suppose $x_1, ..., x_n$ are drawn i.i.d. from the uniform distribution on $(\M, g)$.
The proof of the pointwise convergence of graph Laplacians in (\ref{eq:glconv_euc_unif}) can be broken into two parts \cite{belkin2005towards, belkin2008towards}.

One part follows from standard techniques in concentration of measure. Defining the random variables $X_j^{(\eps)} \vc = k_\eps(x, x_j)$, $Y_j^{(\eps)} \vc = k_\eps(x, x_j)f(x_j)$, we can express $L^{(\eps, n)}f(x)$ as
\[ L^{(\eps, n)}f(x) = \sum_{j = 1}^n f(x) X_j^{(\eps)}-Y_j^{(\eps)}.\]
Since $f(x) X_1^{(\eps)} - Y_1^{(\eps)}, ..., f(x) X_n^{(\eps)}- Y_n^{(\eps)}$ are independent and $\left|f(x) X_j^{(\eps)} - Y_j^{(\eps)}\right| \leq 2\|f\|_\infty$ almost surely,
by Hoeffding's inequality,
\[ \mathbb{P}\left( \left| \frac{L^{(\eps, n)}f(x)}{n} - \E\left[ f(x) X_1^{(\eps)}-Y_1^{(\eps)} \right] \right| > \delta   \right) \leq 2e^{\textstyle-\frac{n\delta^2}{8\|f\|_\infty^2}}. \]
If $\eps_n = 2n^{-\frac{1}{m+2 + \alpha}}$ for some $\alpha > 0$,
\begin{equation}
\label{eq:appofhoeffding}
\mathbb{P}\left( \frac{2}{\eps_n(2\pi \eps_n)^{m/2}}\left| \frac{L^{(\eps_n, n)}f(x)}{n} - \E\left[  f(x) X_1^{(\eps_n)}  - Y_1^{(\eps_n)}\right] \right| > \delta   \right) \leq 2e^{\textstyle-\frac{n\left(\delta (2\pi\eps_n)^{m/2}\eps_n /2\right)^2}{8\|f\|_\infty^2}}
\end{equation}
goes to 0 as $n \to \infty$. Moreover, for $\eps_n = \Omega\left(n^{-\frac{1}{m+2+\alpha}} \right)$ and any $\delta > 0,$
\[ \sum_{n \in \N} 2e^{\textstyle-\frac{n\left(\delta (2\pi\eps_n)^{m/2}\eps_n /2\right)^2}{8\|f\|_\infty^2}}< \infty, \]
so by Borel-Cantelli, $\mathbb{P}\left( E_\delta \right) = 1$,
where $E_\delta$ is the event
\[ E_\delta \vc = \left\{ \exists L \text{ such that } \frac{2}{\eps_n(2\pi \eps_n)^{m/2}}\left| \frac{L^{(\eps_n, n)}f(x)}{n} - \E\left[  f(x) X_1^{(\eps_n)} - Y_1^{(\eps_n)}\right] \right| \leq  \delta \ \forall n \geq L \right\}. \]
Hence
\begin{align*}
    \mathbb{P}\left( \frac{2}{\eps_n(2\pi \eps_n)^{m/2}}\left| \frac{L^{(\eps_n, n)}f(x)}{n} - \E\left[  f(x) X_1^{(\eps_n)} - Y_1^{(\eps_n)} \right] \right|\xrightarrow[]{n \to \infty} 0 \right) = \mathbb{P}\left(\bigcap_{M \in \N} E_{1/M} \right) = 1.
\end{align*} 
Therefore, to prove that 
\begin{equation}
\label{eq:gl_conv_unif}
\frac{2}{\eps_n (2\pi \eps_n)^{m/2}n} L^{(\eps_n, n)} f(x) \to \frac{\Delta_g f(x)}{\mathrm{vol}_g(\M)}
\end{equation}
almost surely as $n \to \infty$, it suffices to prove that
\begin{equation}
\label{eq:consistency}
\lim_{n \to \infty} \frac{2\E\left[ f(x) X_1^{(\eps_n)} - Y_1^{(\eps_n)} \right]}{\eps_n (2\pi \eps_n)^{m/2}} = \frac{\Delta_g f(x)}{\mathrm{vol}_g(\M)}. 
\end{equation}

\begin{remark}
\label{rmk:generalmetricanddist}
We have not used any properties of the uniform distribution or the Euclidean distance so far; (\ref{eq:appofhoeffding}) holds even if $x_1, ..., x_n$ are drawn i.i.d. from any probability distribution on $\M$ and if we replace the Euclidean distance with \emph{any} metric $d$ on $\M$ (see also \cite[Section 4.1]{kileel2021manifold}).
\end{remark}

The second part of proving the pointwise convergence of graph Laplacians (\ref{eq:glconv_euc_unif}) is to prove (\ref{eq:consistency}). When $\M$ is compact, this is done by using the following property of the Euclidean distance and geodesic distance $d_g$ on $\M$ \cite[Lemma 4.3]{belkin2008towards}: there exists a constant $K > 0$ such that
\begin{equation}
    \label{eq:property_euc_dM}
    \|x-y\|^2 \leq d_g^2(x, y) \leq \|x-y\|^2 + Kd_g^4(x,y)
\end{equation}
for all $x, y \in \M$. Therefore, more generally, if $(\M, g)$ is a Riemannian manifold of finite volume, $d$ is a metric on $\M$ which generates the same topology as the topology of $\M$ from the manifold structure, and there exists $K > 0$ such that
\begin{equation}
\label{eq:property_suff}
    d^2(x, y) \leq d_g^2(x, y) \leq d^2(x, y) + K d_g^4(x, y)
\end{equation}
for all $x, y \in \M$, then Belkin and Niyogi's proof of the pointwise convergence of graph Laplacians (\ref{eq:glconv_euc_unif}), \textit{mutatis mutandis}, still holds.
A careful reading of the proof in \cite{belkin2005towards, belkin2008towards} shows that it suffices for (\ref{eq:property_suff}) to hold locally, as we detail in the next section.

\subsection{Sufficient conditions for the pointwise convergence of the graph Laplacian}
\label{subsec:suff}

Suppose $\M$ is a smooth oriented manifold equipped with a $C^{k+1}$ Riemannian metric $g$ for some $k \in \N$ such that $k \geq 3$, and suppose also that $(\M, g)$ has finite volume. Let $\iota \colon \M \to (X, d)$ be a topological embedding; for convenience, we treat $d$ as a metric on $\M$.
Recall the following assumptions mentioned in the introduction:
\firstorderassump*
\assumptwo*

To draw a parallel with (\ref{eq:property_suff}), we prove that if Assumption \ref{assump:firstorder} holds, then the lower bound in (\ref{eq:property_suff}) holds locally.
\begin{lemma}
\label{lemma:upper}
Assumption \ref{assump:firstorder} implies $d(x, y) \leq d_g(x, y)$ for all $x, y \in \M$ such that $d_g(x, y) < \mathrm{inj}_{(\M, g)}(x)$.
\end{lemma}

\begin{proof}
For any $x, y \in \M$ such that $d_g(x, y) < \mathrm{inj}_{(\M, g)}(x)$, there exists a unique $v \in T_x\M$ such that $g_x(v, v) = 1$ and $y = \exp_x(tv)$ for $t = d_g(x, y)$. Define $\gamma \colon [0, t] \to \M$ by $\gamma(s) \vc = \exp_x(sv)$ for $s \in [0, t]$. Since $\gamma$ is a length-minimizing, unit-speed geodesic between $\gamma(s)$ and $\gamma(\widetilde{s})$ for any $s, \widetilde{s} \in [0, t]$,
by Assumption \ref{assump:firstorder}, there exists a function $R \colon [0, \infty) \to [0, \infty]$ such that $R(u)/u \to 0$ as $u \to 0^{+}$ and
\begin{equation}
\label{eq:rem-bound-along-geod}
|d(\gamma(s), \gamma(\widetilde{s})) - |s-\widetilde{s}|| \leq R(|s-\widetilde{s}|) 
\end{equation}
for all $s, \widetilde{s} \in [0, t]$.
Then for any $N \in \N$, the triangle inequality and (\ref{eq:rem-bound-along-geod}) imply that
\[ d(x, y) \leq \sum_{k \in [N]} d(\gamma(s_{k-1}), \gamma(s_k)) \leq N \left( \frac{t}{N} + R\left( \frac{t}{N}\right) \right), \]
where $s_k \vc= \frac{kt}{N}$ for $k = 0, \ldots, N$. Taking $N \to \infty$ gives us
\[ d(x, y) \leq t = d_g(x, y). \]
\end{proof}

Assumption \ref{assump:thirdorder} can be seen as a local version of the upper bound in (\ref{eq:property_suff}). We claim that Assumptions \ref{assump:firstorder} and \ref{assump:thirdorder} allow us to generalize
Belkin and Niyogi's result \cite[Theorem 3.1]{belkin2008towards} on the pointwise convergence of graph Laplacians almost word-for-word.

\begin{theorem}
\label{thm:glconv_unif}
Suppose that Assumptions \ref{assump:firstorder} and \ref{assump:thirdorder} are satisfied, and that $x_1, ..., x_n$ are drawn i.i.d. uniformly from $\M$ (with respect to $g$). Fix any $f \in C^3(\M) \cap L^\infty(\M)$ and $x \in \M.$ If $\eps_n = 2n^{-\frac{1}{m+2 + \alpha}}$ for some $\alpha > 0$, then 
\[\frac{2}{\eps_n (2\pi \eps_n)^{m/2}n} L^{(\eps_n, n)} f(x) \to \frac{\Delta_g f(x)}{\mathrm{vol}_g(\M)}  \]
almost surely as $n \to \infty$.
\end{theorem}

From our discussion at the beginning of the section, to prove Theorem \ref{thm:glconv_unif}, it suffices to prove (\ref{eq:consistency}) with the Euclidean distance replaced by a metric $d$ satisfying Assumptions \ref{assump:firstorder} and \ref{assump:thirdorder}.
For each $\eps > 0$, define the operator $G_\eps$ by\footnote{This is similar to the operator $G_\eps$ as defined in \cite{coifman2006diffusion, lafon2004diffusion}, but we have used a different normalization since it is more convenient for our proof.}
\[ G_\eps f(x) \vc= \frac{1}{(2\pi \eps)^{m/2}}\int_\M k_\eps(x, y)f(y) \ dV_g(y) \]
for any $f \in L^\infty(\M)$, and let
\[ A_\eps(x) \vc = \frac{1}{(2\pi \eps)^{m/2}}\int_\M k_\eps(x, y) \ dV_g(y). \]
Observe that when $x_1$ is drawn from the uniform distribution on $(\M, g)$, we have
\begin{align*}
G_\eps f(x) &= \frac{\mathrm{vol}_g(\M) \E\left[ Y_1^{(\eps)} \right]}{(2\pi \eps)^{m/2}} \\
A_\eps(x) &= \frac{\mathrm{vol}_g(\M) \E\left[ X_1^{(\eps)} \right]}{(2\pi \eps)^{m/2}}. 
\end{align*}
The following proposition implies (\ref{eq:consistency}), which implies Theorem \ref{thm:glconv_unif}.

\begin{proposition}
\label{prop:lapapprox}
    Suppose that Assumptions \ref{assump:firstorder} and \ref{assump:thirdorder} are satisfied.
    Fix any $f \in C^3(\M) \cap L^\infty(\M)$ and $x \in \M$. Then
    \begin{equation}
    \label{eq:lapapprox}
     f(x)A_\eps(x) - G_\eps f(x) = \frac{\eps}{2}\Delta_g f(x) + O(\eps^{3/2})
    \end{equation}
    as $\eps \to 0$.

\end{proposition}
 We provide a non-asymptotic version as well to better understand the benefits of a larger $\eps_0$. Fix any $c \in (0, 1).$ For $x \in \M$, define
\[ r(x) \vc = \min\left( \eps_0, c \  \mathrm{inj}_{(\M, g)}(x) \right). \]
For $x \in \M$ and $R > 0$, define
       \begin{equation}  
       \label{def:beta}
       \beta(x, R) \vc = \inf_{\substack{y \in \M \\ d_g(y, x) \geq R
    }} d(y, x). 
    \end{equation}
Since we assumed that $\M$ is embedded in $(X, d)$, $\beta(x, R)$ is strictly positive for all $x \in \M$ and $R > 0$ (\cref{lemma:betapos}).
\begin{proposition}
    \label{prop:nonasymp}
    Suppose Assumption \ref{assump:firstorder} holds and that Assumption \ref{assump:thirdorder} holds for $K, \eps_0$ such that $\eps_0^2 < \frac{1}{2K}$. Then for any $x \in \M$, $f \in C^3(\M) \cap L^\infty(\M)$ and $\delta > 0$,
    there exists a constant $C > 0$ (dependent on $(\M, g)$, $x$, $f$, $\delta$ and $c$, but not $K$ or $\eps_0$) such that
    \begin{equation}
    \label{eq:nonasymp}
    \left|f(x) A_\eps(x) - G_\eps f(x) - \frac{\eps}{2}\Delta_g f(x) \right| \leq R_1(\eps) + R_2(\eps) + R_3(\eps), 
    \end{equation}
    for all $\eps < \delta$,
    where
    \begin{align*}
        R_1(\eps) &\vc= \frac{2\mathrm{vol}_g\left( \M \setminus B_{r(x)}(x)\right)e^{- \frac{\beta(x, r(x))^2}{2\eps}} \|f\|_\infty}{ \left(2 \pi \eps \right)^{m/2} } \\
        R_2(\eps) &\vc= C\max(1, K)\eps^{3/2} \\
        R_3(\eps) &\vc= \frac{\mathrm{area}(\S^{m-1})\Delta_g f(x)}{2m} \left(\frac{\int_{r(x)}^\infty e^{- \frac{y^2}{2\eps}}y^{m+1} dy}{(2\pi \eps)^{m/2}}\right).
    \end{align*}
    Here $B_{r(x)}(x)\vc = \{y \in \M \colon d_g(y, x) < r(x) \}$ and $\mathrm{area}(\S^{m-1})$ denotes the surface area of $\S^{m-1}$ when embedded as a unit sphere in $\R^m$.
\end{proposition}

\begin{remark}
\label{rmk:eps_0-K assump}
    If Assumption \ref{assump:thirdorder} holds for some $K, \eps_0 > 0$, then without loss of generality, we can take $\eps_0^2 < \frac{1}{2K}$. We assume $\eps_0^2 < \frac{1}{2K}$ for \cref{prop:nonasymp} so that we can use the following inequality in our proof:
\[ d_g^2(x, y) - d^2(x, y) < Kd^4_g(x, y) < K\eps_0^2 d_g^2(x, y) < \frac{d^2_g(x, y)}{2} \]
for all $x, y \in \M$ such that $d_g(x, y) < \eps_0$.
\end{remark}

\begin{remark}
Since we have fixed the embedding $\iota \colon \M \to (X, d)$ (i.e., $K$ and $\eps_0$ are fixed), taking $\eps \to 0$ in Proposition \ref{prop:nonasymp} gives us Proposition \ref{prop:lapapprox}, so it suffices to prove Proposition \ref{prop:nonasymp}. However, if we could vary $\eps_0$  and enforce $\eps_0 \leq \beta \sqrt{\eps}$ for some fixed $\beta > 0$ (while also keeping $(\M, g)$, $c$, $x$ and $f$ fixed), as $\eps \to 0$ we have
     \[ \eps^{m/2}R_3(\eps) \asymp \int_{\eps_0}^\infty e^{-\frac{y^2}{2\eps}}y^{m+1} dy \gtrsim \int_{\beta\sqrt{\eps}}^\infty e^{-\frac{y^2}{2\eps}}y^{m+1} dy = \Theta\left( \eps^{\frac{m+2}{2}} \right), \]
     so $R_3(\eps) = \Omega(\eps)$ as $\eps \to 0$: this is not enough to prove (\ref{eq:lapapprox}) under the alternative assumption that $\eps_0 \leq \beta \sqrt{\eps}$. In terms of samples, if we set $\eps = 2n^{-\frac{1}{m+2 + \alpha}}$ for some $\alpha > 0$ (as in \cref{thm:glconv_unif}), then the condition $\eps_0 \leq \beta \sqrt{\eps}$ is equivalent to
     \[ n \leq \left(\frac{2\beta^2}{\eps_0^2}\right)^{m+2+\alpha}. \]
\end{remark}

The proof of Proposition \ref{prop:nonasymp} is almost the same as \cite[Proposition 4.4]{belkin2008towards}. 
The error $R_1(\eps)$ arises from estimating 
\[f(x) A_\eps(x) - G_\eps f(x) = \frac{1}{(2\pi \eps)^{m/2}} \int_\M k_\eps(x, y) (f(x) - f(y))dV_g(y)\]
with the local integral
\[ I_1(x) \vc = \frac{1}{(2\pi \eps)^{m/2}} \int_{B_{r(x)}(x)} k_\eps(x,y) (f(x) - f(y)) dV_g(y). \]
The error $R_2(\eps)$ arises from computing $I_1(x)$ in normal coordinates and approximating it with the following integral on $T_x\M$:
\[ I_2(x) \vc = -\frac{1}{2(2\pi \eps)^{m/2}} \int_{v \in T_x \M, \|v\| < r(x)} e^{-\frac{\|v\|^2}{2\eps}} \mathrm{Hess}_x f(v, v) dv. \]
Lastly, the error $R_3(\eps)$ arises from approximating $I_2(x)$ with
\[ -\frac{1}{2(2\pi \eps)^{m/2}} \int_{v \in \R^m} e^{-\frac{\|v\|^2}{2\eps}} \mathrm{Hess}_x f(v, v) dv = \frac{\eps}{2} \Delta_g f(x). \]
It is not surprising then that even though $R_2(\eps)$ controls the asymptotics for $\eps$, if $\eps_0 < c \ \mathrm{inj}_{(\M, g)}(x)$ (i.e., $r(x) = \eps_0$), decreasing $\eps_0$ increases $R_1(\eps)$ and $R_3(\eps)$.
For completeness, we provide a full proof of \cref{prop:nonasymp} in Appendix \ref{app:glconv}.
As in \cite{belkin2008towards}, we can extend Theorem \ref{thm:glconv_unif} to non-uniform probability distributions.

\begin{theorem}
\label{thm:glconv_nonunif}
    Suppose that Assumptions \ref{assump:firstorder} and \ref{assump:thirdorder} are satisfied, and that $x_1, ..., x_n$ are drawn i.i.d. from a probability distribution $\mu$ on $\M$ with a density $P \in C^3(\M) \cap L^\infty(\M)$ with respect to $dV_g$. Fix any $f \in C^3(\M) \cap L^\infty(\M)$ and $x \in \M.$ If $\eps_n = 2n^{-\frac{1}{m+2 + \alpha}}$ for some $\alpha > 0$, then 
    \[\frac{2}{\eps_n (2\pi \eps_n)^{m/2}n} L^{(\eps_n, n)} f(x) \to P(x) \Delta_g f(x) - 2g_x( \grad_g f(x), \grad_g P (x) ) \]
    almost surely as $n \to \infty$.
\end{theorem}
\begin{proof}
    As before, define $X_j^{(\eps)} \vc = k_\eps(x, x_j)$, $Y_j^{(\eps)} \vc = k_\eps(x, x_j)f(x_j)$, but now suppose $x_1, ..., x_n$ are drawn i.i.d. from $\mu$.
    By \cref{rmk:generalmetricanddist}, it suffices to show that
    \begin{equation}
    \label{eq:fp_limit}
    \lim_{n \to \infty} \frac{2\E\left[ f(x) X_1^{(\eps_n)}-Y_1^{(\eps_n)} \right]}{\eps_n (2\pi \eps_n)^{m/2}} = P(x) \Delta_g f(x) - 2g_x( \grad_g f(x), \grad_g P(x) ). 
    \end{equation}
    As in \cite[Section 5]{belkin2008towards}, (\ref{eq:fp_limit}) can be proven by applying \cref{prop:lapapprox} to $\widetilde{f}(y) \vc = P(y) \left( f(y) - f(x) \right)$.
\end{proof}

\subsection{Sufficient conditions in terms of differentiability of $d^2$}
\label{subsec:suffdiff}
Now suppose $\M$ is only a smooth oriented manifold (not necessarily equipped with a Riemannian metric) and that $\iota \colon \M \to (X, d)$ is a topological embedding.
A priori, we do not know whether there exists a Riemannian metric $g$ on $\M$ such that Assumptions \ref{assump:firstorder} and \ref{assump:thirdorder} hold for this particular embedding $\iota$.
In this section, we describe some conditions under which
there exists a unique $C^2$ Riemannian metric on $\M$ such that Assumptions
\ref{assump:firstorder} and \ref{assump:thirdorder} are satisfied.

We first recall a well-known result due to Palais \cite{palais1957differentiability}: a $C^2$ Riemannian metric $g$ on $\M$ is uniquely determined by its induced geodesic distance $d_g$. One way to see this (as noted in \cite[p. 161]{villani2009optimaloldandnew}) is as follows: for any smooth path $\gamma \colon (-\delta, \delta) \to \M,$ given $d_g$, we can recover the Riemannian metric $g$ via
\[ \sqrt{g_{\gamma(0)}(\gamma'(0), \gamma'(0))} = \lim_{t \to 0} \frac{d_g(\gamma(0), \gamma(t))}{|t|}, \]
or equivalently,
\begin{equation}
\label{eq:palais}
g_{\gamma(0)}(\gamma'(0), \gamma'(0) ) = \lim_{t \to 0} \frac{d_g^2(\gamma(0), \gamma(t))}{t^2}.
\end{equation}
More succinctly, if $g$ is $C^3$ — which implies that $\exp_x$ is $C^2$ around 0, so $d_g^2(\cdot, x)$ is $C^2$ around $x \in \M$ — we have
\begin{equation}
\label{eq:palais-hess}
g_x = \frac{1}{2}\mathrm{Hess}_x\left(d_g^2(\cdot, x) \right) 
\end{equation}
for all $x \in \M$. 

Suppose now that we are only given a smooth, oriented manifold $\M$ and a topological embedding $\iota \colon \M \to (X, d)$ (but not a Riemannian metric on $\M$). We want to find a Riemannian metric $g$ on $\M$ for which the induced geodesic distance $d_g$ approximates $d$ on $\M$. 
The discussion above suggests that if $d^2$ is $C^2$ around the diagonal $D_{\M}$ of $\M \times \M$, we should consider
\begin{equation}
        \label{eq:hessmetric}
    g_x = \frac{1}{2}\mathrm{Hess}_x\left(d^2(\cdot, x) \right)
\end{equation}
for $x \in \M$. In general, (\ref{eq:hessmetric}) only defines a symmetric,
positive semidefinite bilinear form on $T_x \M$, but it becomes a Riemannian metric when restricted to
\begin{equation}
    \label{eq:def_Mtilde}
\widetilde{\M} \vc = \{ x \in \M \colon g_x \text{ is positive definite} \}, 
\end{equation}
which is an open subset of $\M$, and hence a manifold itself.
\begin{lemma}
\label{lemma:psd}
Suppose $d^2$ is $C^{k+2}$ on a neighborhood of the diagonal $D_\M$ of $\M \times \M$ for some $k \in \Z_{\geq 0}$.
    Then (\ref{eq:hessmetric}) defines a symmetric, positive semidefinite bilinear form on $T_x\M$ for all $x \in \M$, which is $C^k$ on $\M$ in the sense that for any chart $\varphi = \left(x^1, \cdots, x^m \right) \colon U \to V \subset \R^m$,
    \[ p \mapsto g_p \left( \frac{\partial}{\partial x_i} \bigg|_p, \frac{\partial}{\partial x_j} \bigg|_p\right) \]
    is $C^k$ on $U$ for all $i, j \in [m]$.
    In particular, $\widetilde{\M}$ is a manifold and
    $g$ is a $C^k$ Riemannian metric on $\widetilde{\M}$.
\end{lemma}
\begin{proof}
    That (\ref{eq:hessmetric}) defines a symmetric, bilinear form on $T_x\M$ for all $x \in \M$ is a standard property of the Hessian associated to a Levi-Civita connection. 
    Another way to see this is in coordinates: suppose $\varphi = \left(x^1, \cdots, x^m \right) \colon U \to V \subset \R^m$ is a chart on $\M$. 
    For convenience, define $f_p \vc = d^2(\cdot, p)$ for each $p \in U$. Recall from our discussion in \cref{subsec:notation} that since $f_p = d^2(\cdot, p)$ attains a minimum at $p$,
    \[ g_p\left( \frac{\partial}{\partial x_i} \bigg|_p, \frac{\partial}{\partial x_j} \bigg|_p \right) = \frac{1}{2} \mathrm{Hess}_p f_p \left( \frac{\partial}{\partial x_i} \bigg|_p, \frac{\partial}{\partial x_j} \bigg|_p \right) = \frac{1}{2} \frac{\partial^2 f_p}{\partial x_i\partial x_j}(p) \]
    for all $p \in U$. In other words, in coordinates, $\mathrm{Hess}_p f_p$ agrees with $\mathrm{Hess}^{\left(\R^m\right)}_{\varphi(p)} (f_p \circ \varphi^{-1} )$ for all $p \in U$, where, for clarity, we use $\mathrm{Hess}$ for the Hessian as defined on $\M$ (with respect to the Levi-Civita connection of some smooth Riemannian metric) and $\mathrm{Hess}^{\left(\R^m\right)}$ for the typical definition of the Hessian on $\R^m$. This is equivalent to taking the top left (or bottom right, because of the symmetry of $d^2$) $m \times m$ block of $\mathrm{Hess}_{(\varphi(p), \varphi(p))}^{\left(\R^{2m}\right)} \left(\widetilde{d}^2\right)$, where $\widetilde{d}(y, z) \vc = d(\varphi^{-1}(y), \varphi^{-1}(z))$ for $y, z \in V$ (see also \cref{rmk:compare-w-arias}).
    
    The remaining statements can also be seen from representing (\ref{eq:hessmetric}) in coordinates: for any $p \in U$, (\ref{eq:hessmetric}) is $C^k$ around $p$ because $d^2$ is $C^{k+2}$ around $(p, p)$. Since $d^2(\cdot, p)$ attains a minimum at $p$, (\ref{eq:hessmetric}) is positive semidefinite for all $p \in \M$. Therefore, for $p \in U$, $g_p$ is positive definite if and only if
    \[ \det \left( \left[ \frac{\partial^2 f_p}{\partial x_i \partial x_j}(p) \right]_{i, j} \right) > 0. \]
    Since $\det$ is continuous, it follows that $\{ p \in U \colon g_p \text{ is positive definite} \}$ is open, so $\widetilde{\M}$ is an open subset of $\M$ as well. 
    Thus $\widetilde{\M}$ is a manifold, and $g$ is a $C^k$ Riemannian metric on $\widetilde{\M}$.
\end{proof}

The following property also follows from the definition of (\ref{eq:hessmetric}). We state it explicitly since it is important for the results in this section.

\begin{lemma}
\label{lemma:ot2}
    If $d^2$ is $C^2$ on a neighborhood of $D_\M$ and $g$ is defined as in (\ref{eq:hessmetric}), then
    \begin{equation}
    \label{eq:ot2}
    \lim_{t \to 0} \frac{d^2(\gamma(t), \gamma(0))}{t^2} = g_{\gamma(0)}\left( \gamma'(0), \gamma'(0) \right) 
    \end{equation}
    for all $C^1$ paths $\gamma \colon (-\delta, \delta) \to \M$.
\end{lemma}

\begin{proof}
    Fix a chart $\varphi \colon U \to V \subset \R^m$. Without loss of generality, assume that $d^2$ is $C^2$ on $U \times U$. Define $\widetilde{d}(y, z) \vc = d(\varphi^{-1}(y), \varphi^{-1}(z))$ for $y, z \in V$. Fix any $y \in V$, so that by Taylor's theorem
    \[ \widetilde{d}^2(z, y) = (z-y)^TA(z-y) + R(z) \|z-y\|^2  \]
    for $A = \frac{1}{2} \mathrm{Hess}_{y}^{(\R^m)}\widetilde{d}^2(\cdot, y)$ and some continuous $R \colon V \to \R$ such that $R(y) = 0$.
    Therefore, for any
    $C^1$ path $\gamma \colon (-\delta, \delta) \to V$ such that $\gamma(0) = y$,
    \[ \lim_{t \to 0} \frac{\widetilde{d}^2(\gamma(t), \gamma(0))}{t^2} = \lim_{t \to 0} \frac{(\gamma(t) - \gamma(0))^T A(\gamma(t) - \gamma(0)) + R(\gamma(t)) \|\gamma(t) - \gamma(0)\|^2}{t^2} = \gamma'(0)^T A \gamma'(0), \]
    which is (\ref{eq:ot2}) in coordinates.
\end{proof}

\begin{remark}
\label{rmk:compare-w-arias}
    The theoretical analysis of both Isomap and Laplacian-based methods for a compact Riemannian submanifold $\M\subset \R^N$ relies on being able to approximate the geodesic distance on $\M$ from the Euclidean distance, at least locally \cite{bernstein2000graph, belkin2005towards, belkin2008towards}. Therefore, it should not be surprising that some of our assumptions overlap with Arias-Castro and Qiao's assumptions when studying Isomap in \cite{arias2025embedding}. Under their model, $\Theta \subset \R^m$ is open and connected, $X \vc= \{ f_\theta \colon \theta \in \Theta\}$ is a family of probability densities with respect to a fixed measure $\mu$ on $\R^N$, and $d$ is a metric on $X$. They assume that their model is identifiable; in our notation, taking $\M = \Theta$, this amounts to requiring
    \begin{align*}
        \iota \colon \M &\to X \\
        \theta &\mapsto f_\theta
    \end{align*}
    to be injective. They ask for the topology generated by $d$ to be the same as the Euclidean topology on $\Theta$; we require $\iota$ to be a topological embedding. Furthermore, they note that if $d^2$ is $C^2$ on $\Theta \times \Theta$, we can define $A(x) \in \R^{m \times m}$ as the top left $m \times m$ block of $\frac{1}{2} \mathrm{Hess}_{(x, x)}(d^2) \in \R^{2m \times 2m}$ for each $x \in \Theta$; if $A(x)$ is non-singular for all $x \in \Theta$, they then consider the following continuous Riemannian metric on $\Theta$:
    \begin{equation}
    \label{eq:hessmetric-euc}
    \widetilde{g}_x(v, v) \vc = v^T A(x) v .
    \end{equation}
    This is a special case of (\ref{eq:hessmetric}). 
    On the other hand, even in the general case where the manifold $\M$ is not necessarily an open subset of $\R^m$, locally (i.e., by using a coordinate chart) (\ref{eq:hessmetric}) can be viewed as taking the top left $m \times m$ block of $\frac{1}{2} \mathrm{Hess}_{(x, x)}(d^2)$.
\end{remark}

\begin{remark}
\label{rmk:degen-example}
    For an example of when (\ref{eq:hessmetric}) degenerates, consider $\M = (-1, 1)$, $X = (-1, 1)$ with the Euclidean distance $d$, and $\iota \colon \M \to X$ defined by
    $\iota(x) \vc = x^3,$
    so
    \[ d^2(\iota(x), \iota(0)) = x^6,\]
    whose second derivative vanishes at $x = 0$.
\end{remark}
Nonetheless, we can show that $\M \setminus \widetilde{\M}$, 
the subset of $\M$ where (\ref{eq:hessmetric}) degenerates,
cannot contain any nonempty open subsets of $\M$.
\begin{proposition}
\label{prop:no-nonempty-open-subset}
    Suppose $d^2$ is $C^{2}$ on a neighborhood of the diagonal $D_\M$ of $\M \times \M$. Define $g$ as in (\ref{eq:hessmetric}) and $\widetilde{\M}$ as in (\ref{eq:def_Mtilde}). Then $\M \setminus\widetilde{\M}$ does not contain any nonempty open subset of $\M$.    
\end{proposition}

We postpone the proof to \cref{app:no-nonempty-open-subset}. For notational convenience, from now on, we will assume that (\ref{eq:hessmetric}) is positive definite for all $x \in \M$, 
so that $g$ is a Riemannian metric on $\M$. 

\begin{assumption}[Nondegeneracy assumption]
\label{assump:nondegen}
(\ref{eq:hessmetric}) is positive definite for all $x \in \M$.
\end{assumption}
If $g$ as in (\ref{eq:hessmetric}) is a $C^2$ Riemannian metric, we can apply \cref{lemma:ot2} to the geodesics on $(\M, g)$.
\begin{lemma}
\label{prop:pw-approx}
    Suppose that $d^2$ is $C^4$ on a neighborhood of $D_\M$ and that the nondegeneracy assumption holds, so that (\ref{eq:hessmetric}) defines a $C^2$ Riemannian metric $g$ on $\M$. Let $\exp$ denote the exponential map on $(\M, g)$. Then
        \begin{equation}
    \label{eq:ot2-along-geodesics}
        \lim_{t \to 0} \frac{d(x, \exp_x(tv))}{|t|} = \sqrt{g_x(v, v)}
    \end{equation}
    for all $(x, v) \in T\M$.
    Moreover, (\ref{eq:hessmetric}) is the unique $C^2$ Riemannian metric which satisfies (\ref{eq:ot2-along-geodesics}), i.e., if $\widetilde{g}$ is a $C^2$ Riemannian metric with exponential map $\widetilde{\exp}$ and
    \begin{equation}
    \label{eq:ot2-uniq}
    \lim_{t \to 0} \frac{d(x, \widetilde{\exp}_x(tv))}{|t|} = \sqrt{\widetilde{g}_x(v, v)}, 
    \end{equation}
    then $g = \widetilde{g}$.
\end{lemma}
\begin{proof}
    (\ref{eq:ot2-along-geodesics}) follows from applying \cref{lemma:ot2} to the geodesic $\gamma(t) \vc = \exp_x(tv)$:
    \[  \lim_{t \to 0} \frac{d^2(x, \exp_x(tv))}{t^2} = g_x(v, v). \]
    For uniqueness, if a $C^2$ Riemannian metric $\widetilde{g}$ on $\M$ satisfies (\ref{eq:ot2-uniq}), then for any $(x, v) \in T\M$,
        \[ \widetilde{g}_x(v, v) = \lim_{t\to 0} \frac{d^2(x, \widetilde{\exp}_x(tv))}{t^2} = g_x(v, v) \]
    by \cref{lemma:ot2},
    since $\widetilde{\gamma}(t) \vc = \widetilde{\exp}_x(tv)$
    is also a $C^1$ path.
\end{proof}

(\ref{eq:ot2-along-geodesics})
is a pointwise version of Assumption \ref{assump:firstorder};
for Assumption \ref{assump:firstorder} to hold, we would like to bound $|d(x, \exp_x(tv)) - |t||$ \textit{uniformly} over 
all $(x, v) \in T\M$ such that $g_x(v, v) = 1$ and $\gamma(s) \vc= \exp_x(sv)$ is length-minimizing from $\gamma(0) = x$ to $\gamma(t) = \exp_x(tv)$. It turns out that with a few additional regularity assumptions, Assumption \ref{assump:firstorder} and Assumption \ref{assump:thirdorder} both hold.
\begin{theorem}
\label{thm:c4}
    Suppose $d^2$ is $C^5$ on a neighborhood of $D_\M$ and that the nondegeneracy assumption holds. Equip $\M$ with the Riemannian metric $g$ as defined in (\ref{eq:hessmetric}). For each $(x, v) \in T\M$ such that $g_x(v, v) = 1$, define
    \begin{equation}
    \label{eq:defh}
    h_{(x, v)}(t) \vc = d^2(x, \exp_x(tv)) 
    \end{equation}
    for all $t$ where $\exp_x(tv)$ is well-defined.
    Suppose moreover that there exists $\kappa, \rho > 0$
    satisfying the following:
    \begin{enumerate}
    \item for any $x, y \in \M$ such that $d_g(x, y) < \rho$, there exists a length-minimizing geodesic segment $\gamma \colon [0, T] \to \M$ from $\gamma(0) = x$ to $\gamma(T) = y$ such that $d^2$ is $C^4$ on a neighborhood of $\Ima \gamma \times \Ima \gamma$;
    \item if $|t| < \rho$, then the absolute value of the fourth derivative $h^{(4)}_{(x, v)}(t)$ is bounded by $\kappa$:
    \[ \left|h^{(4)}_{(x, v)}(t)\right| < \kappa \]
    for all $(x, v) \in T\M$ such that $g_x(v, v) = 1$ and $h^{(4)}_{(x, v)}(t)$ is well-defined.
    \end{enumerate}
    Then Assumption \ref{assump:firstorder} holds and Assumption \ref{assump:thirdorder} holds for $K = \frac{\kappa}{24}$ and $\eps_0 = \rho$. In addition, (\ref{eq:hessmetric}) is the unique $C^2$ Riemannian metric such that Assumptions \ref{assump:firstorder} and \ref{assump:thirdorder} hold.
\end{theorem}

Before proceeding with the proof of \cref{thm:c4}, we remark that if $\M$ is compact, $d^2$ is $C^7$ on a neighborhood of $D_\M$ and the nondegeneracy assumption is satisfied, then there always exist $\kappa, \rho > 0$ satisfying the required conditions for \cref{thm:c4}.

    \begin{lemma}
        \label{lemma:compactness-lemma}
        Suppose $\M$ is compact, $d^2$ is $C^7$ on a neighborhood of $D_\M$, and the nondegeneracy assumption holds. Then there exist $\kappa, \rho > 0$ satisfying the required conditions for \cref{thm:c4}.
    \end{lemma}
    We postpone the proof of \cref{lemma:compactness-lemma} to \cref{app:compact-lemma-proof}.  \cref{lemma:compactness-lemma} provides us with some examples of manifolds and embeddings which fit into our framework:
    \begin{corollary}
    \label{cor:subman}
        Suppose $\M$ is a compact, smoothly embedded submanifold of either 
        \begin{enumerate}[label=(\alph*)]
        \item a smooth, connected finite-dimensional manifold $\widetilde{\M}$, or
        \item a real Hilbert space $\mathbb{H}$. 
        \end{enumerate}
        
        For (a), if $\widetilde{g}$ is a smooth Riemannian metric on $\widetilde{\M}$ and we take $d$ to be the geodesic distance $d_{\widetilde{g}}$ associated to $\widetilde{g}$, then $g$ as in (\ref{eq:hessmetric}) is the restriction of $\widetilde{g}$ to $\M$, i.e., $(\M, g)$ is a Riemannian submanifold of $(\widetilde{\M}, \widetilde{g}).$
        
        For (b), taking $d$ to be the distance arising from the norm $\| \cdot\|_{\mathbb{H}}$ on $\mathbb{H}$, $g$ as in (\ref{eq:hessmetric}) is the restriction of the inner product on $\mathbb{H}$ to $\M$. 
        
        In either case, the nondegeneracy assumption holds, Assumption \ref{assump:firstorder} holds and Assumption \ref{assump:thirdorder} holds for some $K, \eps_0 > 0$.
    \end{corollary}
    \begin{proof}
    The crux actually lies in the terminology here: by \textit{smoothly embedded submanifold}\footnote{The term \textit{embedded submanifold} is used more often. We add the descriptor \textit{smoothly} to emphasize that $\M$ is not just topologically embedded into $\widetilde{\M}$.} of a smooth manifold $\widetilde{\M}$, we mean that there exists a smooth map $\widetilde{\iota} \colon \M \to \widetilde{\M}$ which is a topological embedding \textit{and an immersion}, i.e., $d_x\widetilde{\iota}$ is injective for all $x \in \M$, or, equivalently, $(\widetilde{\iota} \circ \gamma)'(0) \neq 0$ for any smooth $\gamma \colon (-\delta, \delta) \to \M$ such that $\gamma'(0) \neq 0$. Therefore, for any $x \in \M$, $T_x\M$ can be identified with a subspace of $T_{\widetilde{\iota}(x)}\widetilde{\M}$ via $d_x \widetilde{\iota} \colon T_x\M \to T_{\widetilde{\iota}(x)}\widetilde{\M}$. From here, the fact that (\ref{eq:hessmetric}) is the restriction of $\widetilde{g}$ to $\M$ follows from (\ref{eq:palais}); since $\widetilde{g}$ is a Riemannian metric, the nondegeneracy assumption holds. Since the geodesic distance $d_{\widetilde{g}}$ is smooth around the diagonal of $\widetilde{\M} \times \widetilde{\M}$ (see \cref{subsec:diff-geo}), 
    the remaining statements in \cref{cor:subman} for (a) follow from \cref{thm:c4} and \cref{lemma:compactness-lemma}.
    
    The argument when $\M$ is instead smoothly embedded into a real Hilbert space $\mathbb{H}$ is similar; we refer to \cite{lang1999diffgeo} for more details on differentiability in Banach spaces\footnote{Technically, a function $\widetilde{\gamma} \colon (-\delta, \delta) \to \mathbb{H}$ is differentiable at $0$ if there exists a continuous linear map $\Lambda \colon \R \to \mathbb{H}$ such that
    \[ \frac{\|\widetilde{\gamma}(t) - \widetilde{\gamma}(0) - \Lambda(t)\|_{\mathbb{H}}}{|t|} \to 0 \]
    as $t\to 0$, and $\widetilde{\gamma}'(0)$ refers to the map $\Lambda$ \cite{lang1999diffgeo}. By a slight abuse of notation, we use $\widetilde{\gamma}'(0)$ to denote $\Lambda(1)$ instead, i.e., if $\widetilde{\gamma}$ is differentiable at $0$,
    \[ \widetilde{\gamma}'(0) = \lim_{t \to 0} \frac{\widetilde{\gamma}(t) - \widetilde{\gamma}(0)}{t}. \]
    }. In particular, $\M$ being a smoothly embedded submanifold of a real Hilbert space $\mathbb{H}$ means that there exists a smooth topological embedding $\widetilde{\iota} \colon \M \to \mathbb{H}$ such that $(\widetilde{\iota} \circ \gamma)'(0) \neq 0$ for all smooth $\gamma \colon (-\delta, \delta) \to \M$ such that $\gamma'(0) \neq 0$ \cite[Chapter II, Proposition 2.2]{lang1999diffgeo}, and by (\ref{eq:hessmetric}), we have
    \[ g_{\gamma(0)} (\gamma'(0), \gamma'(0)) = \lim_{t \to 0} \frac{\| (\widetilde{\iota} \circ \gamma)(t) - (\widetilde{\iota} \circ \gamma)(0) \|^2_{\mathbb{H}}}{t^2} = \|(\widetilde{\iota} \circ \gamma)'(0)\|_{\mathbb{H}}^2. \]
    Hence (\ref{eq:hessmetric}) is nondegenerate for case (b) as well, and Assumptions \ref{assump:firstorder} and \ref{assump:thirdorder} hold by \cref{thm:c4} and \cref{lemma:compactness-lemma}.
    \end{proof}
    \begin{remark}
    In \cref{subsec:notation-and-assumptions}, we only required $\iota \colon \M \to (X, d)$ to be a topological embedding, because in general $X$ is not a manifold. Even when $X = \widetilde{\M}$ for some smooth manifold $\widetilde{\M}$ and $d$ is the geodesic distance with respect to some smooth Riemannian metric $\widetilde{g}$ on $\widetilde{\M}$, a topological embedding $\iota \colon \M \to (\widetilde{\M}, d)$ does not imply that $\M$ is a smoothly embedded submanifold of $\widetilde{\M}$, even if $\iota$ is smooth. This is why the example in \cref{rmk:degen-example} has a degeneracy.
    \end{remark}

    We now prove \cref{thm:c4}. The idea is to consider the Taylor expansion of $h_{(x, v)}$ around $0$.

\begin{proof}[Proof of \cref{thm:c4}]
We have already discussed uniqueness in \cref{prop:pw-approx}, so it suffices to prove that Assumptions \ref{assump:firstorder} and \ref{assump:thirdorder} hold.
    Fix any $x, y \in \M$ such that $d_g(x, y) < \rho$. Let $\gamma \colon I \to \M$ be a unit-speed geodesic segment which is length-minimizing from $\gamma(0) = x$ to $\gamma(d_g(x, y)) = y$, where $I = [-\delta, d_g(x, y)]$ for $\delta \in (0, \min(\mathrm{inj}_{(\M, g)}(x),\rho))$ small enough such that $d^2$ is $C^4$ on a neighborhood of $\Ima \gamma \times \Ima \gamma$. 
    Let $v \vc= \gamma'(0).$

    Since $d^2$ is $C^5$ around $D_\M$, $g$ is $C^3$, so geodesics on $(\M, g)$ are $C^4$ (see \cref{subsec:diff-geo}).
    Since $d^2$ is $C^4$ on a neighborhood of $\Ima \gamma \times \Ima \gamma$, the fourth derivative $h_{(x, v)}^{(4)}$ is well-defined on $I$ and by assumption, $\left| h_{(x, v)}^{(4)} \right|$ is uniformly bounded by $\kappa$ on $I$. 
    In addition, we have the following Taylor expansion: for $t \in I$,
    \begin{equation}
    h_{(x, v)}(t) = t^2 + \frac{h^{(3)}_{(x, v)}(0) t^3}{6} + \frac{h^{(4)}_{(x, v)}(s) t^4}{24} 
    \end{equation}
    for some $s$ between 0 and $t$. To prove that Assumption \ref{assump:thirdorder} holds, we want to show that the third order term vanishes. To do this, we claim that, similar to \cref{lemma:upper}, 
        \begin{equation}
        \label{eq:distance_comparison}
        d(\gamma(t_0), \gamma(t_1)) \leq |t_1 - t_0| 
        \end{equation}
    for all $t_0, t_1 \in I$.
    Indeed,
    \begin{align*}
    H \colon I \times I &\to \R \\
    t_0, t_1 &\mapsto d^2(\gamma(t_0), \gamma(t_1)) 
    \end{align*}
    is $C^4$ on the compact set $I \times I$. 
    The absolute values of the third derivatives of $H$ can then be uniformly bounded, so there exists $C > 0$ such that
    \[ \left| d^2(\gamma(t_0), \gamma(t_1)) - (t_1 - t_0)^2 \right| \leq C |t_1 - t_0|^3 \]
    for all $t_0, t_1 \in I$. It follows that for any $t_0, t_1 \in I$,
    \begin{equation}
    \label{eq:diff-from-dist}
    \left| d(\gamma(t_0), \gamma(t_1)) - |t_1 - t_0| \right| \leq \frac{C|t_1-t_0|^3}{d(\gamma(t_0), \gamma(t_1)) + |t_1 - t_0|} \leq C |t_1-t_0|^2, 
    \end{equation}
    and (\ref{eq:distance_comparison}) follows from a similar argument as the proof of \cref{lemma:upper}: for any $N \in \N$ and $t_0, t_1 \in I$ such that $t_0 < t_1$, we can define $s_k \vc = t_0 + \frac{k}{N} |t_1-t_0|$ for $k = 0, \ldots, N$, so
    \begin{align*}
    d(\gamma(t_0), \gamma(t_1)) &\leq \sum_{k \in [N]} d(\gamma(s_{k-1}), \gamma(s_k)) \\
    &\leq \sum_{k \in [N]} |s_k - s_{k-1}| + C|s_k - s_{k-1}|^2  \\
    &= N \left(\frac{|t_1-t_0|}{N} + \frac{C|t_1-t_0|^2}{N^2} \right)
    \end{align*}
    by the triangle inequality and (\ref{eq:diff-from-dist}). Since $N$ can be arbitrarily large, the inequality in (\ref{eq:distance_comparison}) holds.
     Hence
    \[ h_{(x, v)}(t) - t^2 = d^2(\gamma(0), \gamma(t)) - t^2 \leq 0 \]
    for all $t \in I$, so $(h_{(x, v)}(t) - t^2)/t^3 \leq 0$ for $t \in I$ such that $t > 0$ and $(h_{(x, v)}(t) - t^2)/t^3 \geq 0$ for $t \in I$ such that $t < 0$. In particular,
    \[ h_{(x, v)}^{(3)}(0) = \lim_{t \to 0} \frac{6(h_{(x, v)}(t) - t^2)}{t^3} \]
    must be 0.
    Then for any $t \in I$,
    \[ 0 \leq t^2 - h_{(x, v)}(t) \leq \left| \frac{h^{(4)}_{(x, v)}(s) t^4}{24} \right| \leq \frac{\kappa t^4}{24} \]
    for some $s$ between 0 and $t$. Taking $t = d_g(x, y)$ gives us
    \[ 0 \leq d_g^2(x, y) - d^2(x, y) \leq \frac{\kappa d_g^4(x, y)}{24} \]
    for any $x, y \in \M$ such that $d_g(x, y) < \rho$. This implies that Assumption \ref{assump:firstorder} holds (take $R(u) = \frac{\kappa u^3}{24}$ if $u \in [0, \rho),$ and $R(u) = \infty$ for $u \geq \rho$) and that Assumption \ref{assump:thirdorder} holds for $K = \frac{\kappa}{24}$ and $\eps_0 = \rho$.
\end{proof}

We emphasize that uniqueness in \cref{thm:c4} is for a given embedding $\iota \colon \M \to (X, d)$: if we have several different embeddings $\iota_1 \colon \M \to (X_1, d_1), ..., \iota_k \colon \M \to (X_k, d_k)$ which satisfy the assumptions for \cref{thm:c4}, applying (\ref{eq:hessmetric}) may result in \emph{different} Riemannian metrics $g^{(1)}, ..., g^{(k)}$ on $\M$, and therefore the graph Laplacian may converge to a different operator. We provide an example of this in Section \ref{subsec:weightedl1}. 
As observed in \cite{lederman2023manifold}, the learned Riemannian metric may not agree with what the user thinks of as the ``ground truth" Riemannian metric, and in general, this is unavoidable regardless of the metric we use to measure the distance between data points.

\section{Examples}
\label{sec:examples}
It follows from \cref{cor:subman} that \cref{thm:glconv_unif} and \cref{thm:glconv_nonunif} (pointwise convergence of graph Laplacians) applies to smoothly embedded compact submanifolds of Euclidean space and more generally, smoothly embedded compact submanifolds of a real Hilbert space. This is already quite useful, since a number of metrics which have been used in data science and machine learning are Hilbertian, i.e., the associated metric space can be isometrically embedded into a Hilbert space.
Examples of Hilbertian metrics include the Wasserstein-2 distance on $\mathcal{P}_2(\R)$ \cite{peyre2019computational}, the sliced-Wasserstein distance 
\cite{bonneel2015sliced} and the energy distance \cite{szekely2003statistics}\footnote{The energy distance between two measures $\mu, \nu$ on $\R^N$ is a scalar multiple of the homogeneous negative Sobolev norm $\|\mu - \nu \|_{\dot{H}^{-\frac{N+1}{2}}}$\cite{szekely2003statistics, peyre2019computational} and has also been called the sliced-Cram\'er distance \cite{kolourisliced, nadjahi2020statistical}. See also the sliced-Volterra distance \cite{leeb2024metricsrobustnoisedeformations}, which is defined on integrable, compactly supported functions $f, g$ on $\R^N$ and agrees with the energy/sliced-Cram\'er distance if $f, g$ are probability densities.}.

However, not all metrics are Hilbertian. A result from Schoenberg tells us that a metric space $(X, d)$ can be isometrically embedded into a Hilbert space if and only if $y^T [d^2(x_i, x_j)]_{i, j} y \leq 0$ for all $x_1, ..., x_n \in X$ and $y \in \R^n$ such that $\sum_{i} y_i = 0$ \cite{schoenberg1938metric}. In fact, for $d \geq 2$, the Wasserstein-2 distance on $\mathcal{P}_2(\R^d)$ is not Hilbertian \cite{peyre2019computational}. In this section, we provide additional examples of manifolds and embeddings into metric spaces which satisfy the required conditions for \cref{thm:glconv_unif} and \cref{thm:glconv_nonunif}.

\subsection{Weighted $\ell_1$-norm}
\label{subsec:weightedl1}
Consider the weighted $\ell_1$-norm on $\R^2$ given by
\[ \|x\|_{w, 1} = w_1 |x_1| + w_2 |x_2| \]
for some $w_1, w_2 > 0$ and the embedding
\begin{align*}
    \iota \colon \S^1 &\to \R^2 \\
    \theta &\mapsto (\cos\theta, \sin \theta).
\end{align*}
In other words, consider the following metric on $\S^1$: 
\[d(\psi, \theta) \vc= w_1|\cos \psi - \cos \theta| + w_2 |\sin \psi - \sin \theta|\]
for $\psi, \theta \in [0, 2\pi).$
Let $g$ be the standard Riemannian metric on $\S^1$ arising from the restriction of the Euclidean metric on $\R^2$. In \cite[Section 3.7]{kileel2021manifold}, Kileel et al. prove that if $\theta_1, ..., \theta_n$ are drawn i.i.d. uniformly with respect to $g$, then the (appropriately scaled) graph Laplacians constructed using $d$ converge pointwise to a scalar multiple of
\begin{equation}
    \label{eq:weightedl1lim}
    \mathrm{sign}(\cos \theta \sin \theta) \frac{-w_1 |\cos \theta| + w_2 |\sin \theta|}{(w_1 |\sin \theta| + w_2 |\cos \theta|)^4} \frac{d}{d\theta} + \frac{1}{3(w_1 |\sin \theta| + w_2 |\cos \theta|)^3} \frac{d^2}{d\theta^2},
\end{equation}
which they describe as an extrinsic operator. Here we show that we can also understand (\ref{eq:weightedl1lim}) as an intrinsic operator on $\S^1 \setminus \left\{ 0, \frac{\pi}{2}, \pi, \frac{3\pi}{2} \right\}$ with a different Riemannian metric and a non-uniform probability distribution.

To prove the pointwise convergence of graph Laplacians, we check that the assumptions of \cref{thm:c4} hold. The key is to observe that for $(\psi, \theta) \in \left(0, \frac{\pi}{2} \right)^2 \cup \left(\frac{\pi}{2}, \pi \right)^2 \cup \left(\pi, \frac{3\pi}{2} \right)^2 \cup \left(\frac{3\pi}{2}, 2\pi \right)^2$,
\[ \psi, \theta \mapsto d^2(\psi, \theta) = \left\| \left(\cos \psi -\cos\theta, \sin\psi - \sin\theta \right) \right\|_{w, 1}^2 \]
is smooth. We can now compute
\[ \widetilde{g}_\theta \vc= \frac{1}{2}\mathrm{Hess}_\theta\left(d^2(\cdot, \theta)\right) \]
for $\theta \in [0, 2\pi) \setminus \left\{ 0, \frac{\pi}{2}, \pi, \frac{3\pi}{2} \right\}$
by using \cref{lemma:ot2}:
\begin{align*}
\widetilde{g}_\theta \left( \frac{d}{d\theta}, \frac{d}{d\theta} \right) & =  \lim_{\eps \to 0} \frac{\|(\cos (\theta + \eps) - \cos(\theta), \sin(\theta + \eps) - \sin (\theta))\|_{w, 1}^2}{\eps^2} \\
&= \left( \lim_{\eps \to 0} \frac{w_1 |\cos(\theta + \eps) - \cos(\theta)| + w_2 |\sin(\theta + \eps) - \sin(\theta)|}{|\eps|} \right)^2 \\
&= \left( w_1 |\sin \theta| + w_2 |\cos \theta| \right)^2.
\end{align*}
Since $\widetilde{g}_\theta$ is nondegenerate for all $\theta \in [0, 2\pi) \setminus \left\{ 0, \frac{\pi}{2}, \pi, \frac{3\pi}{2} \right\}$, $\widetilde{g}$ is a Riemannian metric on $\S^1 \setminus \left\{ 0, \frac{\pi}{2}, \pi, \frac{3\pi}{2} \right\}$.
Moreover, $\S^1 \setminus \left\{ 0, \frac{\pi}{2}, \pi, \frac{3\pi}{2} \right\}$ has finite volume with respect to $\widetilde{g}$. If $\psi, \theta \in [0, 2\pi) \setminus \left\{ 0, \frac{\pi}{2}, \pi, \frac{3\pi}{2} \right\}$ lie in different connected components, then the geodesic distance $d_{\widetilde{g}}(\psi, \theta)$ is infinite, so 
it suffices to check that the assumptions for \cref{thm:c4} hold on each connected component. Take any $\psi \in \left(0, \frac{\pi}{2} \right)$; the argument for the other quadrants is similar. Define
\begin{align*}
    \gamma \colon \left( - \psi, \frac{\pi}{2}- \psi  \right)&\to \left( 0 , \frac{\pi}{2}\right)\\
    t &\mapsto \psi + t
\end{align*}
and its (signed) length $l \colon \left( - \psi, \frac{\pi}{2}- \psi  \right) \to \R$ by
\begin{equation*}
    l(t) \vc = 
        \int_0^t \sqrt{\widetilde{g}_{\gamma(\theta)} \left( \frac{d}{d\theta}, \frac{d}{d\theta} \right)} d\theta = \int_0^t w_1\sin(\psi + \theta) + w_2 \cos(\psi + \theta) d\theta.
\end{equation*}
for $t \in \left( - \psi, \frac{\pi}{2}- \psi  \right).$
Since $l$ is strictly increasing and smooth on $\left( - \psi, \frac{\pi}{2}- \psi  \right)$, it has a smooth inverse $l^{-1} \colon \Ima l \to \left( - \psi, \frac{\pi}{2}- \psi  \right)$. Reparametrizing $\gamma$ by arc-length (with respect to $\widetilde{g}$) gives us
\[ \widetilde{\gamma} \vc = \gamma \circ l^{-1}, \]
a length-minimizing unit-speed geodesic (with respect to $\widetilde{g}$) that connects $\psi$ to each $\theta \in \left(0, \frac{\pi}{2}\right)$. Moreover, $d^2$ is smooth on $\Ima \widetilde{\gamma} \times \Ima \widetilde{\gamma} = \left(0, \frac{\pi}{2} \right) \times \left(0, \frac{\pi}{2} \right)$, and its derivatives up to fourth order can be uniformly bounded in terms of $w_1$ and $w_2$. Since the first four derivatives of $l^{-1}$ can also be uniformly bounded in terms of $w_1$ and $w_2$, e.g.,
\begin{align*}
\left|(l^{-1})'(s)\right| =\frac{1}{|w_1\sin(\psi + l^{-1}(s)) + w_2 \cos(\psi + l^{-1}(s))|} \leq \frac{1}{\min_{\theta \in \left[0 , \frac{\pi}{2}\right]} |w_1 \sin\theta + w_2 \cos \theta|} 
\end{align*}
for all $s \in \Ima l$, the fourth derivative of $h_{(\psi, \widetilde{\gamma}'(0))}(t) \vc = d^2(\psi, \widetilde{\gamma}(t))$, $t \in \Ima l$ can be bounded in terms of $w_1$ and $w_2$. Hence the assumptions for \cref{thm:c4} hold for some $\kappa, \rho > 0$.

To identify the limiting operator, observe that
the uniform distribution with respect to $g$ is a \textit{non-uniform} probability distribution with respect to $\widetilde{g}$ with probability density
\[ P(\theta) = \frac{1}{2\pi(w_1|\sin \theta| + w_2 |\cos \theta|)}, \]
and by \cref{thm:glconv_nonunif}, the limiting operator of the graph Laplacians constructed with $d$ is given by
\begin{align*}
    P(\theta) \left( \Delta_{\widetilde{g}}f(\theta) - \frac{2}{P(\theta)} \widetilde{g}_\theta \left( \grad_{\widetilde{g}} f, \grad_{\widetilde{g}}P \right) \right)
\end{align*}
for any $\theta \in [0, 2\pi) \setminus \left\{ 0, \frac{\pi}{2}, \pi, \frac{3\pi}{2} \right\}$ and smooth $f \colon \S^1 \to \R$. Using the formula for the Laplacian and gradient in coordinates, we have
\begin{align*}
\Delta_{\widetilde{g}} &= -\frac{1}{w_1 |\sin \theta| + w_2 |\cos \theta|} \frac{d}{d\theta} \left( \frac{1}{w_1 |\sin \theta| + w_2 |\cos \theta|} \frac{d}{d\theta}\right) 
= -(2\pi)^2P(\theta) \frac{d}{d\theta}\left( P(\theta) \frac{d}{d\theta} \right)
\end{align*}
\begin{align*}
    \widetilde{g}_\theta \left( \grad_{\widetilde{g}} f, \grad_{\widetilde{g}}P \right) &= \widetilde{g}_\theta \left( \left(\frac{1}{w_1 |\sin \theta| + w_2 |\cos \theta|}\right)^2\frac{df}{d\theta} \frac{d}{d\theta}, \left(\frac{1}{w_1 |\sin \theta| + w_2 |\cos \theta|}\right)^2 \frac{dP}{d\theta} \frac{d}{d\theta} \right) \\
    &= \left(\frac{1}{w_1 |\sin \theta| + w_2 |\cos \theta|}\right)^2 \frac{df}{d\theta}(\theta) \frac{dP}{d\theta}(\theta) \\
    &= (2\pi)^2P(\theta)^2 \frac{df}{d\theta}(\theta) \frac{dP}{d\theta}(\theta)
\end{align*}
for $\theta \in [0, 2\pi) \setminus \left\{ 0, \frac{\pi}{2}, \pi, \frac{3\pi}{2} \right\}$.
Therefore, we obtain that the limiting operator is
\begin{align*}
    P(\theta) \left( \Delta_{\widetilde{g}}f(\theta) - \frac{2}{P(\theta)} \widetilde{g}_\theta \left( \grad_{\widetilde{g}} f, \grad_{\widetilde{g}}P \right) \right) &= -(2\pi)^2\left((P(\theta))^3\frac{d^2 f}{d\theta^2}(\theta) + 3P(\theta)^2\frac{dP}{d\theta}(\theta) \frac{df}{d\theta}(\theta)\right)
\end{align*}
for $\theta \in [0, 2\pi) \setminus \left\{ 0, \frac{\pi}{2}, \pi, \frac{3\pi}{2} \right\}$,
which, upon computation, is a scalar multiple of (\ref{eq:weightedl1lim}).

\subsection{Wasserstein distance}
\label{subsec:wass}
By \cref{cor:subman}, if $\M$ is compact and smoothly embedded into a smooth, connected finite-dimensional Riemannian manifold $(\widetilde{\M}, \widetilde{g})$, Assumptions \ref{assump:firstorder} and \ref{assump:thirdorder} are satisfied for $(\M, g)$ and $(X, d) = (\widetilde{\M}, d_{\widetilde{g}})$, where $g$ is the restriction of $\widetilde{g}$ to $\M$ and $d_{\widetilde{g}}$ is the geodesic distance on $(\widetilde{\M}, \widetilde{g})$.
As introduced in \cref{subsubsec:wass_riem}, the Wasserstein space $(\mathcal{P}_{2}(\R^N), W_2)$ is equipped with a formal Riemannian structure.
Now suppose that we have a topological embedding $\iota \colon \M \to (\mathcal{P}_{2}(\R^N), W_2)$. Recall that our results in \cref{sec:glconv} depend on the particular embedding $\iota$ of $\M$ into $(\mathcal{P}_{2}(\R^N), W_2)$, but for convenience of notation, we treat $W_2$ as a metric on $\M$. Take any $C^1$ path $\gamma \colon (-\delta, \delta) \to \M$. 
If $W_2^2$ is $C^2$ around the diagonal $D_{\M}$ of $\M\times \M$, then we can compute (\ref{eq:hessmetric}) via \cref{lemma:ot2}: for all $s \in (-\delta, \delta)$,
\begin{equation}
\label{eq:induced-w2-1}
    \frac{1}{2}\mathrm{Hess}_{\gamma(s)}\left(W_2^2(\cdot, \gamma(s)) \right)\left( \gamma'(s), \gamma'(s) \right) = \lim_{t \to 0} \frac{W_2^2\left( \gamma(s+t), \gamma(s) \right)}{t^2} = |(\iota \circ \gamma)'|^2(s)
\end{equation}
 where $|(\iota \circ \gamma)'|$ is the metric derivative of $\iota \circ \gamma$, defined in \cref{subsubsec:wass_riem}. One can show (\cref{lemma:dist-upper-bound}) that $\iota \circ \gamma$ is absolutely continuous. Therefore, as discussed in \cref{subsubsec:wass_riem}, there is a Borel measurable time-dependent vector field $\{v_s\}_{s \in (-\delta, \delta)}$ which satisfies both the continuity equation for $\iota \circ \gamma$ in the distributional sense and $v_s \in T_{\iota(\gamma(s))}\mathcal{P}_2(\R^N)$ for a.e. $s \in (-\delta, \delta)$. Moreover, for this vector field $\{v_s\}_{s \in (-\delta, \delta)}$ (which is uniquely determined for a.e. $s \in (-\delta, \delta)$), we have 
 \begin{equation}
\label{eq:induced-w2}
    \frac{1}{2}\mathrm{Hess}_{\gamma(s)}\left(W_2^2(\cdot, \gamma(s)) \right)\left( \gamma'(s), \gamma'(s) \right)= |(\iota \circ \gamma)'|^2(s) = \|v_s\|_{L^2(\iota(\gamma(s)))}^2
\end{equation}
for a.e. $s \in (-\delta, \delta)$. If the nondegeneracy assumption (Assumption \ref{assump:nondegen}) also holds, then $\M$ can be interpreted (formally) as a Riemannian submanifold of Wasserstein space. In comparison, Hamm et al. consider a bi-Lipschitz embedding $\iota \colon (\M, g) \to \left( \mathcal{P}_{ac}(\Omega), W_2 \right)$ of a Riemannian manifold $(\M, g)$ into the space of absolutely continuous probability measures $\mathcal{P}_{ac}(\Omega)$ on a compact, convex subset $\Omega \subset \R^N$ (along with several other assumptions) \cite{hamm2024manifoldlearningwassersteinspace}. They then (formally) pull back the Riemannian structure from Wasserstein-2 space to $\M$ \cite[Section 2.3]{hamm2024manifoldlearningwassersteinspace}. If $W_2^2$ is $C^2$ on a neighborhood of the diagonal $D_\M$ of $\M \times \M$ and the assumptions in \cite{hamm2024manifoldlearningwassersteinspace} are satisfied, then (\ref{eq:induced-w2}) corresponds with the Riemannian structure which is pulled-back from Wasserstein-2 space in \cite{hamm2024manifoldlearningwassersteinspace}. Similarly, (\ref{eq:induced-w2}) corresponds to the (formal) pull-back of the Wasserstein-2 Riemannian structure in \cite{chen2020optimal}, the Wasserstein information matrix in \cite{li2023wasserstein} and the intrinsic Wasserstein metric in \cite{arias2025embedding}, if $W_2^2$ is $C^2$ on a neighborhood of $D_\M$ and the assumptions in the respective works are satisfied. 

We now provide a few examples of embeddings of compact Riemannian manifolds into $(\mathcal{P}_2(\R^N), W_2)$ which satisfy Assumptions \ref{assump:firstorder} and \ref{assump:thirdorder}.

\begin{example}[Translations]
\label{example:translations}
    This example is similar to \cite[Section 3.3]{hamm2023wassmap} and \cite[Example 2.6]{hamm2024manifoldlearningwassersteinspace}. Fix any $\lambda \in \mathcal{P}_{2}(\R^N)$ which is absolutely continuous with respect to the Lebesgue measure. For each $\theta \in \R^N$, define the translation map $T_\theta \colon \R^N \to \R^N$ by $T_\theta(x) = x + \theta$. Then for any $\theta, \theta' \in \R^N,$ it is known (e.g., \cite[Lemma 3.5]{hamm2023wassmap}) that
    \[ W_2((T_\theta)_{\#} \lambda, (T_{\theta'})_{\#} \lambda) = |\theta - \theta'|, \]
    so if we let $\M$ be a compact, smoothly embedded submanifold of $\R^N$, 
    \begin{align*}
    \iota \colon \M &\to (\mathcal{P}_{2}(\R^N),W_2) \\
    \theta &\mapsto (T_{\theta})_{\#}\lambda
    \end{align*}
    satisfies Assumptions \ref{assump:firstorder} and \ref{assump:thirdorder} with $\M$ inheriting the typical Riemannian metric on $\R^N$.
    \end{example}

    \begin{example}[Dilations] 
    \label{example:dilations}
    This example is similar to \cite[Section 3.4]{hamm2023wassmap} and \cite[Example 2.7]{hamm2024manifoldlearningwassersteinspace}. Again fix any $\lambda \in \mathcal{P}_{2}(\R^N)$ which is absolutely continuous with respect to the Lebesgue measure. For each $\theta = [\theta_j]_{j \in [N]} \in \R^N$ such that $\theta_j > 0$ for all $j = 1, ..., N$, define $D_\theta \colon \R^N \to \R^N$ by $D_\theta(x) = [\theta_jx_j]_{j \in [N]}.$ Then \cite[Lemma 3.7]{hamm2023wassmap} states that
        \[ W_2^2((D_\theta)_{\#}(\lambda), (D_{\theta'})_{\#}(\lambda)) = \sum_{j = 1}^N |\theta_j - \theta'_j|^2 \int_{\R^N}|x_j|^2 d\lambda(x)\]
    for any $\theta, \theta' \in \R^N$ such that $\theta_j, \theta_j' > 0$ for all $j.$
    Therefore, if $\M$ is a compact, smoothly embedded submanifold of $\{ \theta \in \R^N \colon \theta_j > 0 \ \forall j \}$,
    \begin{align*}
    \iota \colon \M &\to (\mathcal{P}_{2}(\R^N),W_2) \\
    \theta &\mapsto (D_{\theta})_{\#}\lambda
    \end{align*}
    satisfies Assumptions \ref{assump:firstorder} and \ref{assump:thirdorder}, but with $\M$ equipped with the restriction of the Riemannian metric
    \[ g\left( \frac{\partial}{\partial x^i}, \frac{\partial}{\partial x^j} \right) = \delta_{ij} \int_{\R^N} |x_j|^2 d\lambda(x) \]
    defined on $\R^N.$ (Here $\left\{\frac{\partial}{\partial x^j}\right\}_{j \in [N]}$ are the standard coordinate vectors on $\R^N$.)
    \end{example}

    \begin{remark}
        For convenience, we have taken $\M$ to be a smooth manifold (without boundary) in this paper, though \cref{example:translations} and \cref{example:dilations} also work if $\M$ is a compact convex subset of $\R^N$, as in \cite[Example 2.6]{hamm2024manifoldlearningwassersteinspace}. If $x$ belongs to the interior of $\M$, then \cref{thm:glconv_unif} and \cref{thm:glconv_nonunif} also hold, following almost the same proof.
    \end{remark}

    \begin{remark}
        In \cite{arias2025embedding}, Arias-Castro and Qiao study location-scale families, which include both translations and dilations of some fixed $\lambda \in \mathcal{P}_2(\R^N)$ which is absolutely continuous with respect to the Lebesgue measure. If $N = 1$, or if $\lambda$ has a spherically symmetric density, they are able to derive a closed-form expression for the squared Wasserstein-2 distance between two elements of the location-scale family \cite[Section 5.1.5]{arias2025embedding}. One can then use these closed-form expressions to check the required smoothness for \cref{lemma:compactness-lemma} and construct examples similar to \cref{example:translations} and \cref{example:dilations}.
    \end{remark}
    
    \begin{example} [Submanifolds of the Bures-Wasserstein manifold] For a fixed $N \in \N$, the space of nondegenerate $N$-dimensional Gaussians can equipped with the Wasserstein-2 Riemannian structure in a way that makes it a finite-dimensional smooth Riemannian manifold $(\widetilde{\M}, \widetilde{g})$ (not just formal) with $d_{\widetilde{g}} = W_2$; this is known as the Bures-Wasserstein manifold, and we refer to \cite{chewi2024statisticaloptimaltransport} for more details (see also \cite{bhatia2019bures}). Hence, by \cref{cor:subman}, if $\M$ is a compact smooth manifold which is smoothly embedded into the Bures-Wasserstein manifold $(\widetilde{\M}, \widetilde{g})$ via $\iota \colon \M \to \widetilde{\M}$, taking $g$ to be the restriction of $\widetilde{g}$ to $\M$ makes it so that Assumptions \ref{assump:firstorder} and \ref{assump:thirdorder} are satisfied for $(\M, g)$ and the embedding $\iota$.
\end{example}

Additionally, in the case $N = 1$, $(\mathcal{P}_{2}(\R), W_2)$ is Hilbertian (see, e.g., \cite[Remark 2.30]{peyre2019computational} or \cite[Proposition 1.18]{chewi2024statisticaloptimaltransport}).

\subsection{The rotating ball (non-)example}
\label{subsec:cursed_example}
We warn that a topological embedding $\iota \colon \M \to X$ will not always satisfy Assumptions \ref{assump:firstorder} and \ref{assump:thirdorder}, even when $\M$ is compact and the ambient space $(X, d)$ is a real Hilbert space. More specifically, if $\iota$ is not smooth, then we cannot view $\M$ as a smoothly embedded submanifold of a real Hilbert space and apply \cref{cor:subman}. In fact, the embedding given in the introduction
\begin{align*}
    \iota \colon \S^1 &\to L^2(\R^2) \\
    \theta &\mapsto \frac{1}{\pi r^2} \ind_{B_r((\cos \theta, \sin \theta))},
\end{align*}
where $r \in (0, 1]$ is fixed,
is an example of this. Let $(\M, g)$ be $\S^1$ with its standard Riemannian metric. Consider a ball $B_1$ of radius $r$ centered at $(1, 0)$ and another ball $B_2$ of radius $r$ centered at $(\cos \theta, \sin \theta)$ for $\theta \in (0, \pi)$. If $B_1 \cap B_2$ is nonempty, the area of $B_1 \cap B_2$ can be computed via
\begin{align*}
\mathrm{area}(B_1 \cap B_2) &= 2\left(r^2 \arccos\left( \frac{\sin\left( \frac{\theta}{2} \right)}{r} \right) - \sin \left( \frac{\theta}{2}\right)\sqrt{r^2 - \sin^2\left( \frac{\theta}{2} \right)} \right) \\
&= 2\left( r^2 \left( \frac{\pi}{2} - \frac{\sin\left( \frac{\theta}{2} \right)}{r} + O\left( \left( \frac{\sin\left( \frac{\theta}{2} \right)}{r} \right)^3 \right) \right) - \sin \left( \frac{\theta}{2}\right)\sqrt{r^2 - \sin^2\left( \frac{\theta}{2} \right)} \right) \\
&= \pi r^2 - 2r\theta + O(\theta^3),
\end{align*}
from which it follows that
\[ \|\iota(\theta) - \iota(0)\|_{L^2}^2 = 2\left( \frac{1}{\pi r^2} \right)^2 \left(\pi r^2 - \mathrm{area}(B_1 \cap B_2) \right) = \Theta\left( \theta \right) \]
as $\theta \to 0^{+}$,
so $\|\iota(\theta) - \iota(0)\|_{L^2} = \Theta(\sqrt{\theta})$ cannot possibly satisfy Assumption \ref{assump:firstorder}. A similar observation was made in \cite{donoho2005image, grimes2003new} while studying Isomap when the ambient space is $L^2(\R^2)$. In contrast, with this embedding $\iota$, the Wasserstein-2 distance
\[ W_2^2(\iota(\theta), \iota(\varphi)) = |(\cos \theta, \sin \theta)-(\cos \varphi, \sin \varphi)|^2, \]
satisfies both Assumption \ref{assump:firstorder} and \ref{assump:thirdorder}; this is equivalent to considering the typical embedding of $\S^1$ as a unit circle in $\R^2$.
Perhaps a bit less obvious is that the $L^1(\R^2)$ norm satisfies both Assumption \ref{assump:firstorder} and \ref{assump:thirdorder} (with respect to $\left(\frac{4}{\pi r}\right)^2g$): for small $\theta > 0$,
\[ \|\iota(\theta) - \iota(0)\|_{L^1} = \frac{2}{\pi r^2}  \left(\pi r^2 - \mathrm{area}(B_1 \cap B_2)   \right) =\frac{2}{\pi r^2}  \left( 2r\theta + O(\theta^3) \right),  \]
so there exists $K, \eps_0 > 0$ such that
\[ \left| \| \iota(\theta) - \iota(0)\|_{L^1} - \frac{4|\theta|}{\pi r} \right| < K |\theta|^3 \]
for all $\theta \in [-\pi, \pi)$ such that $|\theta| < \eps_0$, which implies that there exists $\widetilde{K}, \widetilde{\eps}_0 > 0$ such that
\[ \| \iota(\theta) - \iota(0)\|_{L^1}^2 - \left(\frac{4|\theta|}{\pi r} \right)^2 = \left(  \| \iota(\theta) - \iota(0)\|_{L^1} - \frac{4|\theta|}{\pi r}\right)\left( \| \iota(\theta) - \iota(0)\|_{L^1} + \frac{4|\theta|}{\pi r} \right) \leq \widetilde{K}|\theta|^4 \]
for all $\theta \in [-\pi, \pi)$ such that $|\theta| < \widetilde{\eps}_0$. 
\begin{remark}
\label{rmk:laplace}
Since for all $\theta \in [0, 2\pi)$,
\[  \|\iota(\theta) - \iota(0)\|_{L^1} = \pi r^2 \|\iota(\theta) - \iota(0)\|_{L^2}^2, \]
applying the Laplacian eigenmaps algorithm with the $L^2(\R^2)$ norm and Gaussian kernel $k_\eps$ with parameter $\eps$ is equivalent to applying the Laplacian eigenmaps algorithm with the $L^1(\R^2)$ norm and \textit{Laplace} kernel
\[ \widetilde{k}_{\beta}(x,y) \vc = e^{-\frac{d(x, y)}{2\beta}} \]
with $\beta = \pi r^2\eps.$
We leave the study of the convergence of the graph Laplacians with other kernels to future work, though our numerical experiments in Section \ref{sec:numexp} suggest that the Gaussian kernel exhibits better performance than the Laplace kernel in the Laplacian eigenmaps algorithm.
\end{remark}

\section{Numerical experiments}
\label{sec:numexp}

In this section, we demonstrate the pointwise convergence of the graph Laplacian in \cref{thm:glconv_unif} with numerical experiments, using the rotating ball example from \cref{subsec:cursed_example}. For all experiments, we discretize $[-2, 2] \times [-2, 2]$ into images of size $128 \times 128$ and take $(\M, g)$ to be $\S^1$ with its standard Riemannian metric. Additionally, for convenience, we use the embeddings
\begin{align*}
    \iota_r \colon \S^1 &\to L^2\left( \overline{B_2(0)} \right) \\
    \theta &\mapsto \frac{1}{4r}\ind_{B_r((\cos \theta, \sin \theta))}
\end{align*}
instead of the normalization in (\ref{eq:iota}) so that by our calculations in \cref{subsec:cursed_example}, the $L^1(\R^2)$ norm $\theta, \varphi \mapsto \|\iota_r(\theta) - \iota_r(\varphi)\|_{L^1}$ locally approximates the geodesic distance $d_g$ on $(\M, g)$ for all $r \in (0, 1]$. We confirm this in \cref{fig:exp2_dist}. 
We test pointwise convergence on $f(\theta) \vc= \sin(\theta)$, whose Laplacian can be explicitly computed as $\Delta_{g}f(\theta) = \sin(\theta)$ for all $\theta \in [0, 2\pi)$. Each graph Laplacian constructed from $n$ samples is multiplied by $c_n \vc = \frac{4\pi}{\eps_n(2\pi \eps_n)^{1/2}n}$, where $\eps_n = 2n^{-\frac{1}{m+2 + \alpha}}$ and $\alpha = 0.01$, so that if the assumptions of \cref{thm:glconv_unif} are satisfied, for each $\theta \in [0, 2\pi)$, the limit of $c_nL^{(\eps_n, n)} f(\theta)$ as $n \to \infty$ should be $\Delta_g f(\theta)$ (almost surely).
The code to reproduce the results in this section can be found at \href{https://github.com/lzx23/manifold_learning_metric_spaces}{https://github.com/lzx23/manifold\_learning\_metric\_spaces}.
\subsection{Comparison of different metrics}
\label{subsec:diffmetrics}
For this experiment, we fix $r = 0.75.$
We compare several different metrics which we have already discussed:
\begin{itemize}[label = {--}]
    \item \textbf{$\boldsymbol{L^1(\R^2)}$ norm:} as discussed in \cref{subsec:cursed_example}, 
    \[ d(\theta, \varphi)= \|\iota_r(\theta) - \iota_r(\varphi)\|_{L^1} \]
    locally approximates the geodesic distance on $\M$ with the standard metric.

    \item \textbf{$\boldsymbol{L^2(\R^2)}$ norm:} for our experiments, this refers to the metric induced by the $L^2(\R^2)$ norm \textit{with the following normalization}:
    \[ d(\theta, \varphi) = \sqrt{4r} \|\iota_r(\theta) - \iota_r(\varphi)\|_{L^2} . \]
    This normalization satisfies
    \[ \left( \sqrt{4r} \|\iota_r(\theta) - \iota_r(\varphi)\|_{L^2} \right)^2 = \| \iota_r(\theta) - \iota_r(\varphi)\|_{L^1} \]
    for all $\theta, \varphi \in \S^1$, which allows us to simultaneously investigate the effect of using a Laplace kernel instead of a Gaussian kernel (see \cref{rmk:laplace}).

    \item \textbf{$\boldsymbol{\R^2}$ norm:} this refers to the metric
    \[ d(\theta, \varphi )= \left|(\cos \theta, \sin \theta) - (\cos \varphi, \sin \varphi) \right|. \]
    As discussed in the introduction, for each $\theta, \varphi \in \S^1$, this agrees with the Wasserstein-2 metric between $\mu_\theta,\mu_\varphi \in \mathcal{P}_2(\R^2)$ which have densities given by normalizing $\iota_r(\theta)$ and $\iota_r(\varphi)$, respectively.
\end{itemize}
\begin{figure}[h!]
    \centering
    \includegraphics[width=0.7\linewidth]{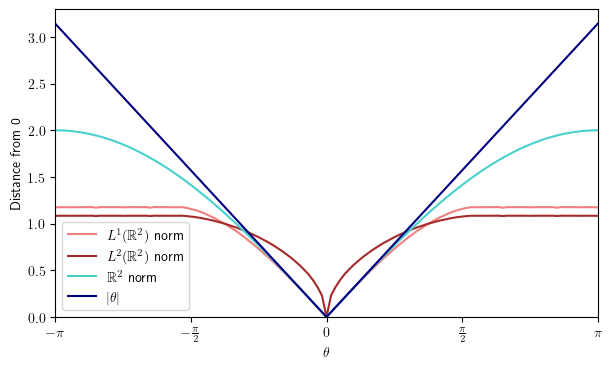}
    \caption{Comparison of different metrics.}
    \label{fig:exp1_dist}
\end{figure}
In \cref{fig:exp1_dist}, we plot $\theta \mapsto d(\theta, 0)$ for $\theta \in [-\pi,\pi]$ for each of the aforementioned metrics. We also include the geodesic distance, which for $\theta \in [-\pi, \pi]$ is given by $d_g(\theta, 0) = |\theta|$, for comparison.
\cref{fig:exp1_dist} confirms our calculations in \cref{subsec:cursed_example}: the $L^1(\R^2)$ norm approximates $d_g(\theta, 0) = |\theta|$ around 0, whereas the $L^2(\R^2)$ norm behaves like $\sqrt{|\theta|}$ around 0. 

Although the $L^1(\R^2)$ and $\R^2$ norm both approximate the geodesic distance $d_g$ around 0, \cref{fig:exp1_dist} suggests that the $\R^2$ norm provides a better approximation. 
In \cref{fig:exp1_pwconv}, we draw $n$ samples $\theta_1, ..., \theta_n$ i.i.d. from the uniform distribution on $(\M, g)$ and plot the discrete Laplacian of $f$ constructed with each of the metrics in \cref{fig:exp1_dist} (and the Gaussian kernel). We see that the discrete Laplacian with the $\R^2$ norm converges to $\Delta_g f$ much faster than the discrete Laplacian with $L^1(\R^2)$ norm, whereas the discrete Laplacian with the $L^2(\R^2)$ norm (equivalently, the discrete Laplacian with a Laplace kernel and $L^1(\R^2)$ norm) does not converge to $\Delta_g f$. 
\begin{figure}[h!]
\hspace{80pt}
    \includegraphics[width=0.8\linewidth]{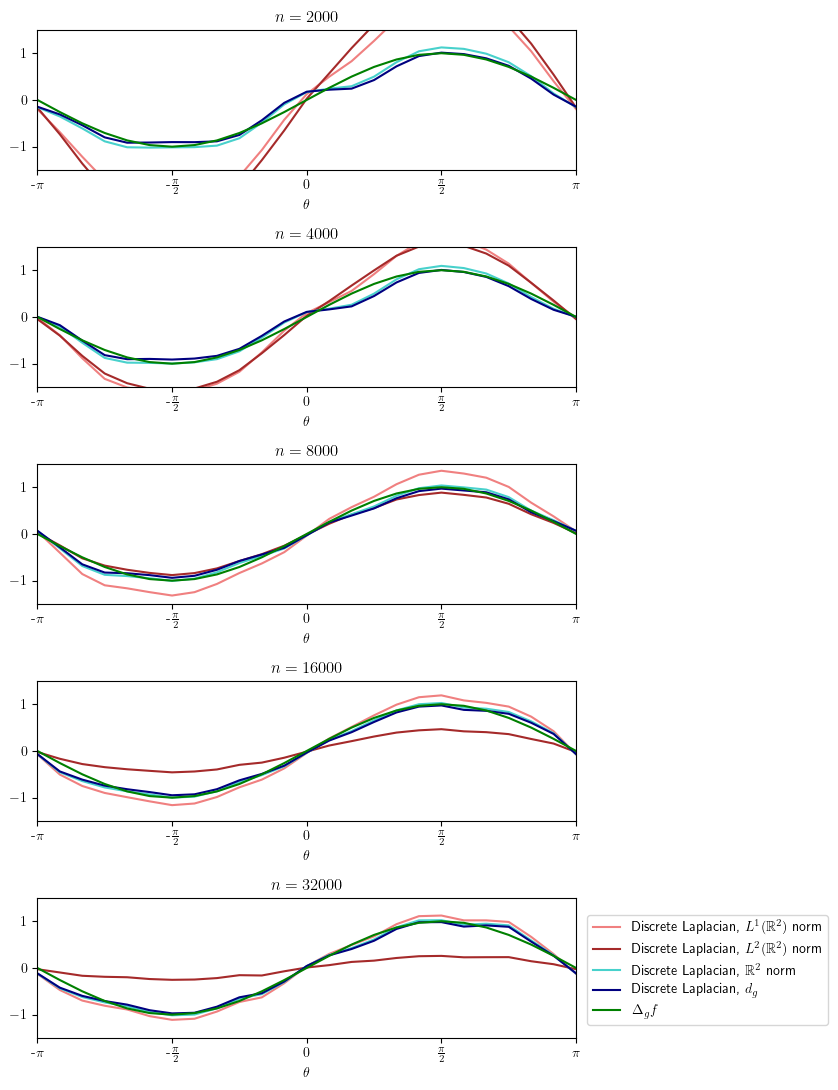}
    \caption{Comparison of the discrete Laplacian for different metrics.}
    \label{fig:exp1_pwconv}
\end{figure}

Although we did not study spectral convergence in this work, we are ultimately motivated by the use of various metrics in the Laplacian eigenmaps and diffusion maps algorithms. Therefore, we also plot the Laplacian eigenmap embeddings obtained from the first two non-trivial eigenvectors of the normalized graph Laplacian in \cref{fig:exp1_eigen}. Again we see that using metrics which better approximate $d_g$ results in embeddings which better capture the geometry of $(\M, g)$, especially at lower sample sizes.

\begin{figure}[h!]
    \centering
    \includegraphics[width=.7\linewidth]{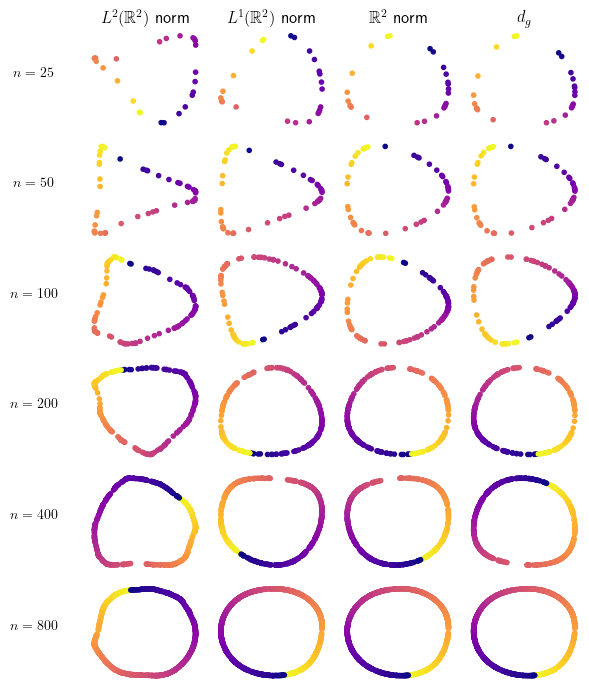}
    \caption{Laplacian eigenmap embeddings using the first two non-trivial eigenvectors $v^{(1)}, v^{(2)}$ of the normalized graph Laplacian $\mathcal{L}^{(\eps_n, n)}$ constructed with different metrics. More specifically, $\theta_1, ..., \theta_n$ are once again drawn i.i.d. from the uniform distribution on $\S^1$. The first two non-trivial eigenvectors $v^{(1)}, v^{(2)}$ of $\mathcal{L}^{(\eps_n, n)}$ with the specified metric are computed, and each sample $\theta_i$ is mapped to $\left(v^{(1)}_i, v^{(2)}_i\right)$ and colored by the true angle $\theta_i$.}
    \label{fig:exp1_eigen}
\end{figure}

\subsection{Changing radii}

To obtain finer control over the quality of the approximation of $d_g,$ we now consider varying the radius $r$ of our rotating ball. Indeed, we have the following lemma.

\begin{lemma}
    \label{lemma:incradii}
    If $0 < r_1 < r_2 \leq 1$, then $\|\iota_{r_1}(\theta) - \iota_{r_1}(0)\|_{L^1} \leq \|\iota_{r_2}(\theta) - \iota_{r_2}(0)\|_{L^1}$ for any $\theta \in [-\pi, \pi]$, and the inequality is strict for $\theta \neq 0$.
\end{lemma}

We prove \cref{lemma:incradii} in \cref{app:incradii}.
As a direct corollary, if $r_1 < r_2$, then $\|\iota_{r_2}(\theta) - \iota_{r_2}(0)\|_{L^1}$ is a better approximation of $d_g(\theta, 0)$ than $\|\iota_{r_1}(\theta) - \iota_{r_1}(0)\|_{L^1}$ in the following sense:

\begin{corollary}
    Suppose $0 < r_1 < r_2 \leq 1$, and that $\theta, \varphi \mapsto \|\iota_{r_1}(\theta) - \iota_{r_1}(\varphi)\|_{L^1}^2$ satisfies Assumption \ref{assump:thirdorder} for some $K > 0$ and $\eps_0 \in (0, \pi)$, i.e., 
    \[ |\theta|^2 - \|\iota_{r_1}(\theta) - \iota_{r_1}(0)\|_{L^1}^2 < K |\theta|^4 \]
    for all $|\theta| < \eps_0$. Then there exists $\widetilde{\eps_0} > \eps_0$ such that
    \[ |\theta|^2 - \|\iota_{r_2}(\theta) - \iota_{r_2}(0)\|_{L^1}^2 < K |\theta|^4 \]
    for all $|\theta| < \widetilde{\eps_0}.$
\end{corollary}
\begin{proof}
    This follows from \cref{lemma:incradii} and the fact that $\theta \mapsto \|\iota_{r}(\theta) - \iota_{r}(0)\|_{L^1}^2$ is continuous for any fixed $r$. If $0 < r_1 < r_2 \leq 1$ and 
    \[ |\theta|^2 - \|\iota_{r_1}(\theta) - \iota_{r_1}(0)\|_{L^1}^2 < K |\theta|^4 \]
    for all $|\theta| < \eps_0$, then, since both sides are continuous,
    \[ |\theta|^2 - \|\iota_{r_1}(\theta) - \iota_{r_1}(0)\|_{L^1}^2 \leq K |\theta|^4 \]
    for all $|\theta| \leq \eps_0$. In particular, for $\theta \in \{-\eps_0, \eps_0\}$,
     \[ |\theta|^2 - \|\iota_{r_2}(\theta) - \iota_{r_2}(0)\|_{L^1}^2 < |\theta|^2 - \|\iota_{r_1}(\theta) - \iota_{r_1}(0)\|_{L^1}^2 \leq K |\theta|^4 \]
     where the first inequality is strict by \cref{lemma:incradii}. Hence there exists $\widetilde{\eps_0} > \eps_0$ such that
     \[ |\theta|^2 - \|\iota_{r_2}(\theta) - \iota_{r_2}(0)\|_{L^1}^2 < K |\theta|^4 \]
    for all $|\theta| < \widetilde{\eps_0}.$

\end{proof}

\begin{figure}[h!]
    \centering
    \includegraphics[width=0.7\linewidth]{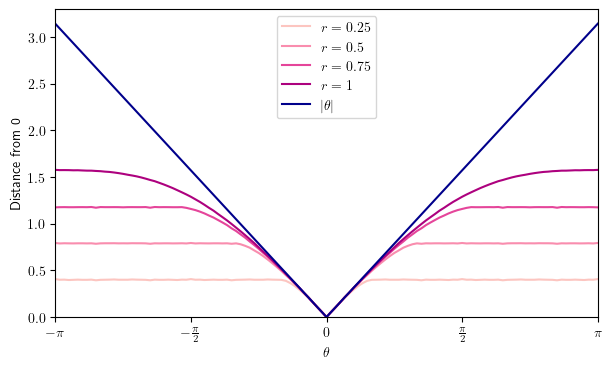}
    \caption{The $L^1(\R^2)$ norm $\|\iota_r(\theta) - \iota_r(0)\|_{L^1}$ with balls of varying radii $r$.}
    \label{fig:exp2_dist}
\end{figure}

\begin{figure}[h!]
    \hspace{35pt}
    \includegraphics[width=0.8\linewidth]{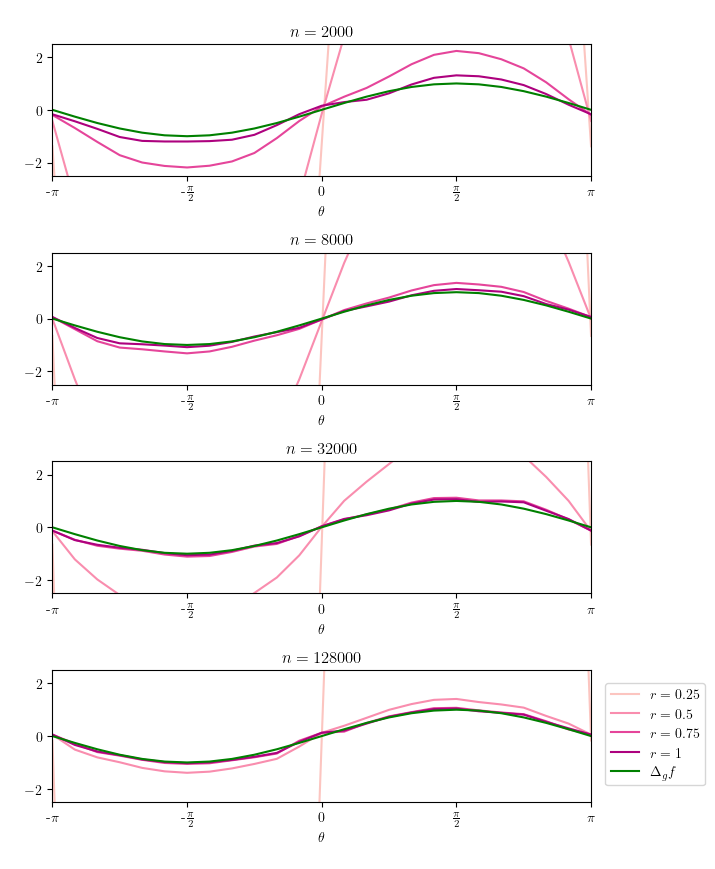}
    \caption{Comparison of the discrete Laplacian with the metric induced by the $L^1(\R^2)$ norm $\theta, \varphi \mapsto \|\iota_r(\theta) - \iota_r(\varphi)\|_{L^1}$ for different values of $r$.}
    \label{fig:exp2_pwconv}
\end{figure}

We confirm in \cref{fig:exp2_dist} that the $L^1(\R^2)$ norm $\|\iota_r(\theta) - \iota_r(0)\|_{L^1}$ locally approximates the geodesic distance $d_g(\theta, 0) = |\theta|$ for each $r$, but that larger $r$ result in better approximations. 

As in \cref{subsec:diffmetrics}, we draw $\theta_1, ..., \theta_n$ i.i.d. from the uniform distribution on $(\M, g)$.
We compare the resulting discrete Laplacians in \cref{fig:exp2_pwconv} and see that a larger radius $r$, which gives us a better approximation of $d_g$, also results in faster convergence of the discrete Laplacians. Similarly, in \cref{fig:exp2_eigen}, we observe that the Laplacian eigenmap embeddings capture the geometry of $(\M, g)$ much better for larger $r$, and this effect is most pronounced for lower sample sizes.

\begin{figure}[h!]
    \centering
    \includegraphics[width=.7\linewidth]{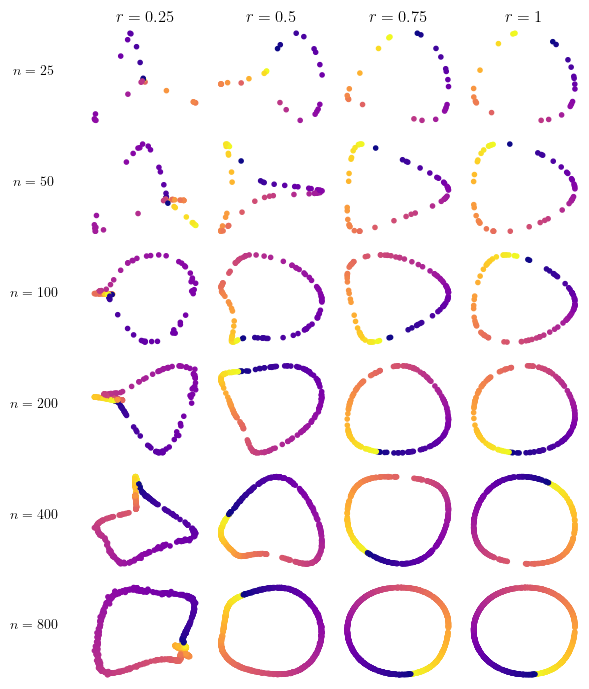}
    \caption{Laplacian eigenmaps embedding, again using the first two non-trivial eigenvectors of the normalized graph Laplacian $\mathcal{L}^{(\eps_n, n)}$. The procedure is the same as \cref{fig:exp1_eigen}, except that we use the metrics $\theta, \varphi \mapsto \|\iota_r(\theta) - \iota_r(\varphi)\|_{L^1}$ for $r = 0.25, 0.5, 0.75, 1.$}
    \label{fig:exp2_eigen}
\end{figure}

\section{Conclusion}
In this paper, we provided sufficient conditions for the pointwise convergence of graph Laplacians when a manifold $\M$ is embedded in a metric space $(X, d).$ We also included non-asymptotic upper bounds on the bias error and demonstrated numerically that due to limited sample sizes, the choice of metric can have a sizable impact on the convergence of graph Laplacians in practice.

Our work suggests several directions for future research. One direction would be to generalize other results from the Euclidean case, e.g., \cite{garcia2020error, cheng2022eigen, hein2005graphs, coifman2005geometric, coifman2006diffusion, garcia2020error, cheng2022eigen,ting2010analysis, calder2022improved, cheng2022knnselftune, cheng2024bi}. Of particular interest for manifold learning is obtaining non-asymptotic bounds for spectral convergence similar to \cref{prop:nonasymp} and determining whether they are tight. Another direction would be to study the stability of the Laplacian eigenmap embedding if we use an approximation of the metric $d$ instead of $d$ itself in the graph Laplacian. For example, in practice, it may be preferable to use an entropy-regularized approximation of the Wasserstein distance \cite{cuturi2013sinkhorn} or a linearization of the Wasserstein-2 distance \cite{wang2013linear, letrouit2024gluing, greengard2022linearization}, since computing the pairwise Wasserstein-2 distance between many measures can be computationally expensive. Lastly, for an embedding $\iota \colon \M \to (\mathcal{P}_2(\R^N), W_2)$, it would be of interest to characterize when $W_2^2$ is regular around the diagonal of $\M \times \M$ and to better understand whether this is a good assumption for various practical applications.

\section*{Acknowledgments}
A.S. and L.X. were supported in part by AFOSR FA9550-23-1-0249, the Simons Foundation Math+X Investigator Award, NSF DMS 2009753, NSF DMS 2510039 and NIH/NIGMS R01GM136780-01. L.X. was supported in part by NSF Graduate Research Fellowship DGE-2039656 and DGE-2444107. Any opinions, findings,
and conclusions or recommendations expressed in this material are those of the authors
and do not necessarily reflect the views of the National Science Foundation. We thank Ruiyi Yang and Roy Lederman for helpful discussions. We also thank the anonymous reviewers for their comments, which greatly improved the manuscript.

\clearpage

\bibliographystyle{abbrv}
\bibliography{refs, gl_refs}
\newpage
\appendix

\section{Review of Riemannian geometry}
    \label{subsec:diff-geo}
    We review some concepts from Riemannian geometry that we use in the paper. We refer to standard textbooks in Riemannian geometry, e.g., \cite{do1992riemannian,petersen, lee-riem-mani}, for more details; these textbooks usually assume that the Riemannian metric $g$ is $C^\infty$, though several results also hold when $g$ is just $C^2$. We say a Riemannian metric $g$ on $\M$ is $C^k$ for $k \in \Z_{\geq 0}$ if for any chart $\varphi = (x^1, \ldots, x^m) \colon U \to V \subset \R^m$,
\begin{equation}
\label{eq:metric-in-coords}
g_{ij} \vc = g\left( \frac{\partial}{\partial x_i}, \frac{\partial}{\partial x_j} \right) 
\end{equation}
is $C^k$ on $U$ for all $i, j \in [m]$. 
We will focus on when $k \geq 2$, since in this setting, the existence, uniqueness and regularity of geodesics (with respect to the Levi-Civita connection) follows from standard ODE theory. 
See \cite{samann2018geodesics} for results when $g$ is not $C^2$.

To see this, fix a chart $\varphi = (x^1, \ldots, x^m) \colon U \to V$, define $g_{ij}$ as in (\ref{eq:metric-in-coords}) and define $g^{ij}$ by $[g^{ij}] = [g_{ij}]^{-1}$. Let $I \subset \R$ be an open interval.
If $g$ is $C^k$ for $k \geq 2$, then the Christoffel symbols
\[ \Gamma_{ij}^l = \frac{1}{2} \sum_{p =1}^m g^{pl} \left( \frac{\partial g_{jp}}{\partial x_i} + \frac{\partial g_{pi}}{\partial x_j} - \frac{\partial g_{ij}}{\partial x_p} \right) \]
are $C^{k-1}$, and a $C^2$ path $\gamma \colon I \to U$ is a \textit{geodesic} if and only if the geodesic equation
\begin{equation}
\label{eq:geodesic-eq}
\frac{d^2 \gamma^l}{dt^2} + \sum_{i, j \in [m]} (\Gamma_{ij}^l \circ \gamma) \frac{d\gamma^i}{dt} \frac{d \gamma^j}{dt} = 0, 
\end{equation}
where $\gamma^i \vc = x^i \circ \gamma$,
holds for all $l \in [m]$. (In Assumption \ref{assump:firstorder} and \cref{thm:c4}, we use the terminology \textit{geodesic segment}; following \cite{do1992riemannian}, a geodesic segment is the restriction of a geodesic $\gamma \colon I \to \M$ to some compact interval $[a, b] \subset I$.) Upon rearranging (\ref{eq:geodesic-eq}), we see that finding a geodesic $\gamma \colon I \to U$ such that $\gamma(0) = p, \gamma'(0) = v$ for $p \in U, v = \sum_{l} v^l \frac{\partial}{\partial x_l} \big|_p \in T_p\M$ is equivalent to 
finding a $C^2$ solution $(\gamma, \widetilde{\gamma}) \colon I \to TU$ to 
the initial value problem
\begin{equation}
\label{eq:geo-flow}
\begin{cases}
    \frac{d\gamma^l}{dt} = \widetilde{\gamma}^l & \\
    \frac{d\widetilde{\gamma}^l}{dt} = -\sum_{i, j} (\Gamma_{ij}^l \circ \gamma) \widetilde{\gamma}^i \widetilde{\gamma}^j & \text{ \ \ \ \ \ for }l \in [m] \\
    \gamma^l(0) = x^l(p), \widetilde{\gamma}^l(0) = v^l& 
    \end{cases}.
\end{equation}
For a fixed initial condition $p \in U, v \in T_p\M$, since the Christoffel symbols $\Gamma_{ij}^l$ are $C^{k-1}$ for $k \geq 2$ (hence locally Lipschitz), by Picard-Lindel\"of, there exists a unique solution $(\gamma, \widetilde{\gamma}) \colon \left(T_{-}, T_{+}\right) \to TU$ to (\ref{eq:geo-flow}), where $\left(T_{-}, T_{+}\right)$ is the maximal interval on which the solution exists. Since the Christoffel symbols are $C^{k-1}$, $\widetilde{\gamma}$ is $C^k$, so our geodesic $\gamma$ is $C^{k+1}$ on $(T_{-}, T_{+})$. 

If we allow our initial conditions to vary, i.e., we want to solve
\begin{equation}
\label{eq:geo-flow-2}
\begin{cases}
    \frac{\partial \gamma^l}{\partial t}(t, p, v) = \widetilde{\gamma}^l(t, p, v) & \\
    \frac{\partial \widetilde{\gamma}^l}{\partial t} = -\sum_{i, j} (\Gamma_{ij}^l \circ \gamma) \widetilde{\gamma}^i \widetilde{\gamma}^j & \text{ \ \ \ \ \ for }l \in [m] \\
    \gamma^l(0, p, v) = x^l(p), \widetilde{\gamma}^l(0, p, v) = v^l& 
    \end{cases},
\end{equation}
then results from ODE theory \cite[Chapter V]{hartman2002ordinary} imply that the solution $(\gamma, \widetilde{\gamma})$ is $C^{k-1}$ on its domain of existence $D \vc= \{ (t, p, v): p \in U, v \in T_p\M, T_{-}(p, v) < t < T_{+}(p, v) \}$. Since $T_{-}$ is upper semicontinuous and $T_{+}$ is lower semicontinuous \cite[Chapter V, Theorem 2.1]{hartman2002ordinary}, $D$ is open. For any $p \in U$, $(1, p, 0) \in D$, so the \textit{exponential map}
\[ \exp_p(v) \vc= \gamma(1, p, v) \]
is well-defined and $C^{k-1}$ on a neighborhood of the zero section of $T\M$. 

The exponential map has several nice properties, which can be found in a standard textbook on Riemannian geometry (e.g., \cite{do1992riemannian,petersen, lee-riem-mani}). We list some of the ones that we use in this paper:
\begin{itemize}[label = {--}]
\item The path $\eta(t)\vc = \exp_p(tv)$ is the unique geodesic such that $\eta(0) = p, \eta'(0) = v.$ The speed $t \mapsto \sqrt{g_{\eta(t)}(\eta'(t), \eta'(t))}$ along a geodesic $\eta$ is constant.
    \item One can check (e.g., \cite[Proposition 5.5.1]{petersen}) that for any $p \in \M$,
$d\exp_p|_{0}$ is invertible (in fact, it is the canonical isomorphism from $T_0T_p\M$ to $T_p\M$),
so by the inverse function theorem, there exists an open neighborhood of $0$ on which $\exp_p$ is a $C^{k-1}$ diffeomorphism. Hence the \textit{injectivity radius}
\[\mathrm{inj}_{(\M, g)}(p) \vc = \sup \{ \delta > 0 \colon \exp_p \text{ is a } C^{k-1}\text{-diffeomorphism on } B_\delta(0) \subset T_p \M \}\]
is strictly positive for all $p \in \M$. Here, for $\delta > 0$, $B_\delta(0) \subset T_p(\M)$ denotes the set $\{v \in T_p(\M) \colon \sqrt{g_p(v, v)} < \delta \}.$
The use of $\exp_p$ as a coordinate chart around $p$ is known as \textit{normal coordinates} around $p$; we use this in the proof of \cref{prop:nonasymp}.
\item For any $p \in \M$, if $\sqrt{g_p(v, v)} < \mathrm{inj}_{(\M, g)}(p)$, then 
$\eta(t) \vc = \exp_p\left(tv\right), t \in [0, 1]$ is the unique (up to reparametrization) length-minimizing piecewise $C^1$ path from $p$ to $\exp_p(v),$ so
\begin{equation}
\label{eq:geo-dist}
d_g^2(p, \exp_p(v)) = g_p(v, v).
\end{equation}
Moreover, $\eta$ must be length-minimizing between $\eta(s)$ and $\eta(t)$ for any $s, t \in [0, 1]$, so
\[ d_g^2(\eta(s), \eta(t)) = |t-s|^2 g_p(v, v) \]
for all $s, t \in [0, 1]$ such that $s < t$.
    \item We also have that for any $p \in \M$, there exists an open neighborhood $V$ of $(p, 0) \in T \M$ such that $F(p, v) \vc= (p, \exp_p(v))$
     is a $C^{k-1}$ diffeomorphism on $V$, so the local inverse $F^{-1} \colon F(V) \to V$ is a $C^{k-1}$ diffeomorphism on an open neighborhood of $F(p, 0) = (p,p)$. Together with (\ref{eq:geo-dist}), this implies that $d_g^2$ is $C^{k-1}$ on a neighborhood of the diagonal of $\M\times \M$.
\end{itemize}

    \section{Proofs for \cref{sec:glconv}}

\subsection{Proof of \cref{prop:nonasymp}}
\label{app:glconv}

\begin{proof}[Proof of Proposition \ref{prop:nonasymp}]
    For this proof, when $\lesssim$ and big-O notation are used, the multiplicative constant can depend on $(\M, g)$, $x$, $f$, $\delta$ and $c$, but not $K$ and $\eps_0$.
    The proof is almost the same as \cite{belkin2008towards}. 

    \textbf{1. Approximation of integrals on $\M$ by an integral over a geodesic ball.}
    We first approximate $f(x)A_\eps(x) - G_\eps f(x)$ with the integral
    \[ I_1(x) \vc = \frac{1}{(2\pi \eps)^{m/2}} \int_{B_{r(x)}(x)} k_\eps(x,y) (f(x) - f(y)) dV_g(y). \]
    This gives us an approximation error of
        \begin{align*}
        \bigg|  f(x) A_\eps(x) - G_\eps f(x) - I_1(x) \bigg| 
        &= \frac{1}{(2\pi \eps)^{m/2}}
        \left|\int_{\M \setminus B_{r(x)}(x)} k_\eps(x, y) (f(x)-f(y)) \ dV_g (y) \right| \\
        &\leq \frac{2\mathrm{vol}_g\left( \M \setminus B_{r(x)}(x)\right)e^{- \frac{\beta(x, r(x))^2}{2\eps}} \|f\|_\infty}{(2\pi \eps)^{m/2}}.
    \end{align*}
    The right-hand side is precisely $R_1(\eps).$ 
    
    \textbf{2. Integrate in normal coordinates.}
    Since $r(x) < \mathrm{inj}_{(\M, g)}(x)$, we can evaluate $I_1(x)$ in normal coordinates:
    \[ I_1(x) = \frac{1}{(2\pi \eps)^{m/2}} \int_{\|v\| < r(x)} k_\eps(x, \exp_x(v))(f(x) - f(\exp_x(v)) \sqrt{|\det [g(v)]|} dv,\]
    where $\|v\| \vc = \sqrt{g_x(v, v)}$ for $v \in T_x\M$.
    We would like to approximate the integral on the right hand side with
    \[I_2(x) \vc =  -\frac{1}{2(2\pi \eps)^{m/2}}\int_{\|v\| < r(x)} e^{-\frac{\|v\|^2}{2\eps}} \mathrm{Hess}_xf(v, v) dv.\]
    Recall that we assumed that the Riemannian metric $g$ is $C^{k+1}$ for $k \geq 3$, and since $r(x) \leq c \ \mathrm{inj}_{(\M, g)}(x) < \mathrm{inj}_{(\M, g)}(x)$, the exponential map $\exp_x$ belongs to $C^3\left(\overline{\widetilde{B}_{r(x)}(0)}\right)$, where $\widetilde{B}_{r(x)}(0) \vc= \{ v \in T_x\M \colon \|v\| < r(x) \}$.
    Since $f \in C^3(\M)$, we can use the Taylor expansions
    \begin{equation}
    \label{eq:fapprox}
    f(\exp_x(v))-f(x) = \langle \grad_g f(x), v\rangle + \frac{1}{2}\mathrm{Hess}_x f(v, v) + O(\|v\|^3)
    \end{equation}
    \begin{equation}
    \label{eq:jacapprox}
    \sqrt{|\det [g(v)]|} = 1 - \frac{\mathrm{Ric}_x(v, v)}{6} + O(\|v\|^3)
    \end{equation}
    for $v \in T_x\M$ such that $\|v\| < r(x)$, where $\mathrm{Ric}_x$ denotes the Ricci curvature tensor at $x$. Moreover, by Assumption \ref{assump:thirdorder}, 
    \begin{equation*}
    \label{eq:kernelest}
    e^{- \frac{\|v\|^2}{2\eps}} \leq k_\eps(x, \exp_x(v)) \leq e^{ - \frac{\|v\|^2}{2\eps}} e^{\frac{K\|v\|^4}{2\eps} }
    \end{equation*}
    for all $v \in T_x\M$ such that $\|v\| < r(x)$. By the mean value theorem and convexity of the exponential function,
    \[ e^t \leq 1 + te^t \]
    for all $t \geq 0$, so for all $v \in T_x\M$ such that $\|v\| < r(x)$,
    \[ e^{\frac{K\|v\|^4}{2\eps}} \leq 1 + \left(\frac{K\|v\|^4}{2\eps} \right)e^{\frac{K\|v\|^4}{2\eps}} \leq 1 + \left(\frac{K\|v\|^4}{2\eps} \right) e^{\frac{\|v\|^2}{4\eps}}, \]
    where the second inequality follows from the assumption $\eps_0^2 < \frac{1}{2K}$ (see \cref{rmk:eps_0-K assump}).
    Therefore, we have
    \[ e^{- \frac{\|v\|^2}{2\eps}} \leq k_\eps(x, \exp_x(v)) \leq e^{-\frac{\|v\|^2}{2\eps}} + \left(\frac{K\|v\|^4}{2\eps} \right)e^{-\frac{\|v\|^2}{4\eps}}\]
    and in particular,
    \small
    \begin{align*}
        \bigg| \int_{\|v\| < r(x)} \left(k_\eps(x, \exp_x(v)) - e^{-\frac{\|v\|^2}{2\eps}}\right)( f(x) - f(\exp_x(v))) \sqrt{|\det [g(v)]|} dv \bigg| &\leq \frac{K}{2\eps} \int_{\|v\| < r(x)} e^{-\frac{\|v\|^2}{4\eps}} O(\|v\|^5)dv \\
        &\lesssim \frac{K}{\eps} \int_{\R^m} e^{-\frac{\|v\|^2}{4\eps}} \|v\|^5 dv \\
        &\lesssim K \eps^{\frac{m+3}{2}}.
    \end{align*}
    \normalsize
    Plugging in the Taylor expansions (\ref{eq:fapprox}) and (\ref{eq:jacapprox}),
    \small
    \begin{align*}
        \int_{\|v\| < r(x)}&  e^{-\frac{\|v\|^2}{2\eps}}( f(x) - f(\exp_x(v))) \sqrt{|\det [g(v)]|} dv \\
        &= - \int_{\|v\| < r(x)} e^{-\frac{\|v\|^2}{2\eps}} \left( \langle \grad_g f(x), v\rangle + \frac{1}{2}\mathrm{Hess}_x f(v, v) + O(\|v\|^3) \right) \left( 1 - \frac{\mathrm{Ric}_x(v, v)}{6} + O(\|v\|^3) \right) dv \\
        &= -\frac{1}{2}\int_{\|v\| < r(x)} e^{-\frac{\|v\|^2}{2\eps}} \mathrm{Hess}_xf(v, v) dv + O\left( \eps^{\frac{m+3}{2}} \right)\\
        &= (2\pi \eps)^{m/2}I_2(x) + O\left( \eps^{\frac{m+3}{2}} \right),
    \end{align*}
    \normalsize
    where
    \[ \int_{\|v\| < r(x)} e^{-\frac{\|v\|^2}{2\eps}} \langle \grad_g f(x), v\rangle \left( 1 - \frac{\mathrm{Ric}_x(v, v)}{6} \right) dv = 0 \]
    by symmetry. It follows that there exists $C > 0$ such that
    \[ \left| I_1(x) - I_2(x) \right| \leq C\max(1, K)\eps^{3/2} =\vc R_2(\eps).  \]

    \textbf{3. Approximation of an integral on a ball with an integral over $\R^m$.}
    By Lemma \ref{lemma:traceint},
        \[ I_2(x) = \frac{1}{2(2\pi \eps)^{m/2}}\left( \frac{\mathrm{area}(\S^{m-1})\Delta_g f(x)}{m} \right)\int_0^{r(x)} e^{-\frac{y^2}{2\eps}}y^{m+1} dy. \]
    Now observe that
        \begin{align*}
        \frac{\mathrm{area}(\S^{m-1})}{m} \int_0^{r(x)} e^{-\frac{y^2}{2\eps}}y^{m+1} dy &=  \frac{\mathrm{area}(\S^{m-1})}{m} \left(\int_0^\infty e^{-\frac{y^2}{2\eps}}y^{m+1} dy - \int_{r(x)}^\infty e^{-\frac{y^2}{2\eps}}y^{m+1} dy \right)\\
        &= \frac{1}{m}\int_{\R^m} e^{-\frac{\|v\|^2}{2\eps}}\|v\|^2 dv - \frac{\mathrm{area}(\S^{m-1})}{m} \int_{r(x)}^\infty e^{-\frac{y^2}{2\eps}}y^{m+1} dy \\
        &= (2\pi \eps)^{m/2} \eps - \frac{\mathrm{area}(\S^{m-1})}{m} \int_{r(x)}^\infty e^{-\frac{y^2}{2\eps}}y^{m+1} dy,
    \end{align*}
    where the last line is obtained by computing the variance of a multi-dimensional Gaussian. It follows that
    \begin{align*}
        \left| I_2(x) - \frac{\eps}{2} \Delta_g f(x) \right| \leq R_3(\eps),
    \end{align*}
    and (\ref{eq:nonasymp}) follows from the triangle inequality.

\end{proof}

\subsection{Proof of \cref{prop:no-nonempty-open-subset}}
\label{app:no-nonempty-open-subset}
The idea is that if (\ref{eq:hessmetric}) degenerates on a nonempty open set, we can find a continuously differentiable path $\gamma \colon [0, \delta] \to \M$ such that $\gamma(0) \neq \gamma(\delta)$, but $\gamma'(t) \neq 0$ and $g_{\gamma(t)}\left( \gamma'(t), \gamma'(t) \right) = 0$ for all $t \in [0, \delta]$. Then we can use the following lemma to show that $d(\gamma(0), \gamma(\delta)) = 0$, which is a contradiction. 
\begin{lemma}
\label{lemma:dist-upper-bound}
Suppose $d^2$ is $C^{2}$ on a neighborhood of the diagonal $D_\M$ of $\M \times \M$.
    For $g$ defined as in (\ref{eq:hessmetric}) and any $C^1$ path $\gamma \colon [0, \delta] \to \M$,
    \begin{equation}
    \label{eq:dist-upper-bound}
    d(\gamma(0),\gamma(\delta)) \leq \int_0^\delta \sqrt{g_{\gamma(t)}(\gamma'(t), \gamma'(t))} dt.
    \end{equation}
\end{lemma}
\begin{proof}
    Let $\varphi \colon U \to V \subset \R^m$ be a chart on $\M$ such that $d^2$ is $C^2$ on $U \times U$ and $V$ is convex.
    It suffices to prove (\ref{eq:dist-upper-bound}) when $\Ima \gamma \subset U$; the general case follows by covering $\Ima \gamma$ with a finite number of such charts and using the triangle inequality.
    Define
    $\widetilde{d} \colon V \times V \to \R$
    by
    \[ \widetilde{d}(y,z) \vc= d(\varphi^{-1}(y), \varphi^{-1}(z)) \]
    for all $y, z \in V$. Let $A(y, z)$ be the top left $m \times m$ block of $\frac{1}{2}\mathrm{Hess}^{(\R^{2m})}_{(y, z)}\left(\widetilde{d}^2\right)$ for $y, z \in V$, so that
    it suffices to prove that
    \begin{equation}
        \widetilde{d}(\gamma(0), \gamma(\delta)) \leq \int_0^\delta \sqrt{\gamma'(t)^T A(\gamma(t), \gamma(t)) \gamma'(t)} dt 
    \end{equation}
    for all $C^1$ paths $\gamma \colon [0, \delta] \to V$. 

    Let $\gamma \colon [0, \delta] \to V$ be a $C^1$ path. By the triangle inequality, for any $N \in \N$, 
    \begin{equation}
    \label{eq:upper-bound-dist}
    \widetilde{d}(\gamma(0), \gamma(\delta)) 
    \leq \sum_{k = 0}^{N-1} \widetilde{d}\left( \gamma\left( t_k \right),\gamma\left( t_{k+1} \right)  \right),
    \end{equation}
    where $t_k \vc= \frac{k\delta}{N}$.
    Now observe that if we fix any $y, z \in V$ and define $\widetilde{\gamma}(t) \vc = (1-t)y + tz$ and $f(t) \vc = \widetilde{d}^2(\widetilde{\gamma}(t), y)$ for $t \in [0,1]$, then by Taylor's theorem with the Lagrange remainder,
    \[ \widetilde{d}^2(z, y) = \frac{1}{2} \frac{\partial^2 f}{\partial t^2}(\alpha) = (z-y)^T A((1-\alpha)y + \alpha z,y)(z-y) \]
    for some $\alpha \in [0, 1]$. In particular, fixing any $u, v \in [0, \delta]$ such that $u< v$, we can additionally use the mean value theorem to obtain
    \begin{align*}
    \widetilde{d}^2(\gamma(u), \gamma(v)) &= (\gamma(v) - \gamma(u))^T A((1-\alpha)\gamma(u) + \alpha \gamma(v), \gamma(u))(\gamma(v) - \gamma(u)) \\
    &= (v-u)^2 \gamma'(s)^T A((1-\alpha)\gamma(u) + \alpha \gamma(v), \gamma(u)) \gamma'(s)
    \end{align*}
    for some $\alpha \in [0, 1]$ and $s \in (u, v)$. Therefore
\[\sum_{k = 0}^{N-1} \widetilde{d}\left( \gamma\left( t_k \right),\gamma\left( t_{k+1} \right) \right) = \frac{\delta}{N} \sum_{k = 0}^{N-1} \sqrt{\gamma'(s_k)^T A\left((1-\alpha_k)\gamma\left( t_k \right) + \alpha_k \gamma\left( t_{k+1} \right), \gamma\left( t_k \right)  \right) \gamma'(s_k)}\]
    for some $\alpha_k \in [0, 1]$ and $s_k \in (t_k, t_{k+1})$. 
    Since
    \[ W \vc = \{ ((1-\alpha)\gamma(t_1) + \alpha \gamma(t_2), \gamma(t_1)) \colon \alpha \in [0, 1], t_1, t_2 \in [0, \delta] \} \]
    is compact and $A$ is continuous on $W$, $A$ is uniformly continuous on $W$. Since $\|\gamma'\|$ is continuous on $[0, \delta]$, there exists $C > 0$ such that $\|\gamma'(t)\| < C$ for all $t \in [0, \delta]$, and so $\|\gamma(v) - \gamma(u)\| < C|v-u|$ for all $u, v \in [0, \delta]$. Together with the uniform continuity of $A$ on $W$, this implies that for any $\eps > 0$, there exists $M \in \N$ such that for all $N \geq M$,
    \[ \|A\left((1-\alpha_k)\gamma\left( t_k \right) + \alpha_k \gamma\left( t_{k+1} \right), \gamma\left( t_k \right)  \right) - A(\gamma(s_k), \gamma(s_k)) \|_{op} < \eps \]
    for any $\alpha_k \in [0, 1]$, $s_k \in \left(t_k, t_{k+1}\right)$ and $k = 0, ..., N-1$.
    Here $\|\cdot\|_{op}$ denotes the operator norm. Hence, for $N \geq M$,
    \begin{align*} 
    \sum_{k = 0}^{N-1} \widetilde{d}\left( \gamma\left( t_k \right),\gamma\left( t_{k+1} \right) \right) &\leq \frac{\delta}{N} \sum_{k = 0}^{N-1} \sqrt{\gamma'(s_k)^T A(\gamma(s_k), \gamma(s_k)) \gamma'(s_k) + \eps \|\gamma'(s_k)\|^2} \\
    &\leq \frac{\delta}{N} \sum_{k = 0}^{N-1} \sqrt{\gamma'(s_k)^T A(\gamma(s_k), \gamma(s_k)) \gamma'(s_k) + \eps C^2 }
    \end{align*}
    for some $s_k \in \left( t_k, t_{k+1}\right)$,
    which upon substituting into (\ref{eq:upper-bound-dist}) gives us
    \[ \widetilde{d}(\gamma(0), \gamma(\delta)) \leq \frac{\delta}{N} \sum_{k = 0}^{N-1} \sqrt{\gamma'(s_k)^T A(\gamma(s_k), \gamma(s_k)) \gamma'(s_k) + \eps C^2 }. \]
    Since $t \mapsto \gamma'(t)^T A(\gamma(t), \gamma(t)) \gamma'(t)$ is continuous on $[0, \delta]$, the lemma statement follows from taking $N \to \infty$ and observing that $\eps > 0$ was arbitrary.
\end{proof}

We now prove \cref{prop:no-nonempty-open-subset}.

\begin{proof}[Proof of \cref{prop:no-nonempty-open-subset}]
    For $k = 0, ..., m-1$, define
    \[ M_k \vc = \{ x \in \M \colon \mathrm{rank}(g_x) \leq k \},\]
    so that $\M \setminus \widetilde{\M} = M_{m-1}$. We will show by induction that $M_k$ does not contain any nonempty open subset $U \subset \M$ for any $k = 0, ..., m-1$.

    For the base case, suppose for contradiction that there exists a nonempty open $U \subset \M$ such that $U \subset M_0$. Take any $C^1$ path $\gamma \colon [0, \delta] \to U$ such that $\gamma(0) \neq \gamma(\delta)$. Since $U \subset M_0$,
    $g_{\gamma(t)}(\gamma'(t), \gamma'(t)) = 0$
    for all $t \in [0, \delta]$, so \cref{lemma:dist-upper-bound} implies $d(\gamma(0), \gamma(\delta)) = 0$, a contradiction.

    Now suppose $M_{k-1}$ does not contain a nonempty open set for some $k \in \{1, ..., m-1\}$. If $M_k$ does contain a nonempty open set $U$, then there must exist some $p \in U \cap (M_k \setminus M_{k-1})$.
    We claim that there exists an open neighborhood $U'$ around $p$ and a continuous vector field $F \colon U' \to T\M$ such that $g_q(F(q), F(q)) = 0$ for each $q \in U'$. Then by Peano's existence theorem, there exists a $C^1$ path $\gamma \colon [0, \delta] \to U'$ such that $\gamma'(t) = F(\gamma(t))$ for all $t \in [0, \delta]$. \cref{lemma:dist-upper-bound} then implies that $d(\gamma(0), \gamma(t)) = 0$ for all $t \in [0, \delta]$, a contradiction.

    For completeness, we detail the construction of such a vector field $F$.\footnote{This construction is a modification of an answer given by an anonymous user on Stack Exchange \cite{1203782}.} Making $U$ smaller if necessary, take
    a chart $\varphi = (x^1, ..., x^m) \colon U \to V$ around $p$. Since $g_p$ is a symmetric, positive semidefinite bilinear form of rank $k$, without loss of generality, we can assume that
    \[ g_p \left( \frac{\partial}{\partial x_i} \bigg|_p, \frac{\partial}{\partial x_j} \bigg|_p \right) = \lambda_i \delta_{ij} \]
    for $\lambda_1 \geq \cdots \geq \lambda_k > 0 = \lambda_{k+1} = \cdots = \lambda_m$. Define $B \colon U \to \R^{k \times k}$ by taking $B(q)$ to be the top left $k \times k$ block of $\left[ g_q\left( \frac{\partial}{\partial x_i}\big|_q, \frac{\partial}{\partial x_j}\big|_q \right) \right]_{i, j}$ for each $q \in U$.
    Since $d^2$ is $C^2$ on a neighborhood of $D_{\M}$, $g$ is continuous, so there exists an open neighborhood $U'$ of $p$ such that $B(q)$ is invertible for all $q \in U'$. But for all $q \in U'$, $g_q$ has rank at most $k$, since $q \in M_k$. Therefore for each $q \in U',$ the $k+1$ to $m$th rows of $\left[ g_q\left( \frac{\partial}{\partial x_i}\big|_q, \frac{\partial}{\partial x_j}\big|_q \right) \right]_{i, j}$ are linear combinations of its first $k$ rows, and together with the symmetry of $\left[ g_q\left( \frac{\partial}{\partial x_i}\big|_q, \frac{\partial}{\partial x_j}\big|_q \right) \right]_{i, j}$, we obtain
    \[ \left[ g_q\left( \frac{\partial}{\partial x_i}\bigg|_q, \frac{\partial}{\partial x_j}\bigg|_q \right) \right]_{i, j} = \begin{bmatrix}
    B(q) & B(q)\Lambda(q)^T \\
    \Lambda(q) B(q) & \Lambda(q) B(q)\Lambda(q)^T
    \end{bmatrix}
    \]
    for some $\Lambda(q) \in \R^{(m-k)\times k}$. We can then take 
    \begin{align*}
        F \colon U' &\to \R^m \\
        q &\mapsto \begin{bmatrix}
        -\Lambda(q)^T v\\
        v
    \end{bmatrix}
    \end{align*}
    for any fixed $v \in \R^{m-k}$. Since $g$ is continuous and $B(q)$ is invertible for all $q \in U'$, $\Lambda$ is continuous on $U'$, so $F$ is continuous on $U'$.

    \end{proof}

\subsection{Proof of \cref{lemma:compactness-lemma}}
\label{app:compact-lemma-proof}
\begin{proof}[Proof of \cref{lemma:compactness-lemma}] 
    Since $d^2$ is $C^7$ around $D_\M$ and the nondegeneracy assumption holds, (\ref{eq:hessmetric}) defines a $C^5$ Riemannian metric $g$, whose associated exponential map $\exp$ is $C^4$ on an open neighborhood $\mathcal{E} \subset T\M$ of the zero section of $T\M$ (see \cref{subsec:diff-geo}). By the inverse function theorem, 
    \begin{align*}
        F \colon \mathcal{E} &\to \M \times \M \\
        (p, v) &\mapsto (p, \exp_p(v))
    \end{align*}
    has a local $C^4$ inverse around $(p, p)$ for each $p \in \M$. Fix any $p \in \M$. 
    Recall that a subset $V \subset \M$ is \textit{strongly convex} if for any $x, y \in \overline{V}$, there exists a unique (up to reparametrization) length-minimizing geodesic segment $\gamma \colon [0, T] \to \M$ from $\gamma(0) = x$ to $\gamma(T) = y$, and $\gamma((0, T))$ is contained in $V$ \cite[p. 74]{do1992riemannian}. For $r > 0$, let 
    \[ B_r(p) \vc = \{ x \in \M \colon d_g(x, p) < r \}, \]
    i.e., $B_r(p)$ is an open ball of radius $r$ with respect to the geodesic distance $d_g$, not $d$.
    We claim that for small enough $r > 0$, $B_r(p)$ is strongly convex; although \cite{do1992riemannian} works with smooth Riemannian manifolds, the argument in \cite[Chapter 3, Section 4]{do1992riemannian} holds here as well. Therefore we can take $r > 0$ such that there exists a $C^4$ inverse $F^{-1}$ for $F$ on an open neighborhood of $\overline{B_r(p) \times B_r(p)}$, $d^2$ is $C^4$ on a neighborhood of $\overline{B_r(p) \times B_r(p)}$ and $B_r(p)$ is strongly convex. 
   
   Since $\M$ is compact, we can cover $\M$ with finitely many balls of the form $B_{r/2}(p)$, where $r$ satisfies the conditions in the previous sentence. For any $x \in B_{r/2}(p)$ and $y \in \M$ such that $d_g(x, y) < \frac{r}{2}$, we have $y \in B_r(p)$, so there exists a length-minimizing geodesic segment $\gamma \colon [0, T] \to B_r(p)$ from $\gamma(0) = x$ to $\gamma(T) = y$ by the strong convexity of $B_r(p)$, and $d^2$ is $C^4$ on a neighborhood of $\Ima \gamma \times \Ima \gamma$ since $\Ima \gamma \subset B_r(p).$
   Thus, it suffices to prove that there exists $\widetilde{\kappa} > 0$ such that
   \[ \left|h_{(x, v)}^{(4)}(t)\right| < \widetilde{\kappa} \]
   for all $x \in B_{r/2}(p)$, unit-norm $v \in T_x\M$ and $|t| < \frac{r}{2}$.
    
    This follows almost immediately from observing that the map $H((x, v), t) \vc = d^2(x, \exp_x(tv)) = d^2(F(x, tv))$ is well-defined and $C^4$ on a neighborhood of the compact set $S \vc = \{ (x, v) \in T\M \colon x \in \overline{B_{r/2}(p)}, g_x(v, v) = 1 \} \times \left[ - \frac{r}{2}, \frac{r}{2} \right]$, so
        \[ \sup_{((x, v), t) \in S} \left|h^{(4)}_{(x, v)}(t) \right| < \infty. \]
\end{proof}

\section{Proof of \cref{lemma:incradii}}
\label{app:incradii}

\begin{proof}[Proof of \cref{lemma:incradii}]
    It suffices to prove the statement for $\theta \in (0, \pi]$ and when $\supp \iota_{r_1}(\theta) \cap \supp \iota_{r_1}(0)$ is non-empty. The latter condition is equivalent to the condition
    \[ \left| \frac{\sin \left( \frac{\theta}{2} \right)}{r_1} \right| < 1. \]
    Using the computations from \cref{subsec:cursed_example}, for any $r \in (0, 1]$ and $\theta \in (0, \pi]$ such that $\supp \iota_{r}(\theta) \cap \supp \iota_{r}(0) \neq \varnothing$,
    \begin{align}
        \left\| \iota_r(\theta) - \iota_r(0) \right\|_{L^1} &= \frac{1}{2r} \left( \pi r^2 - 2 \left(r^2 \arccos\left( \frac{\sin\left( \frac{\theta}{2} \right)}{r} \right) - \sin \left( \frac{\theta}{2}\right)\sqrt{r^2 - \sin^2\left( \frac{\theta}{2} \right)} \right) \right) \\
        &= \frac{\pi r}{2} - r \arccos \left( \frac{\sin\left( \frac{\theta}{2} \right)}{r} \right) + \sin \left( \frac{\theta}{2} \right) \sqrt{1 - \left( \frac{\sin\left( \frac{\theta}{2} \right)}{r} \right)^2}.
        \label{eq:iotar_L1}
    \end{align}
    Taking $f(r, \theta) \vc = \left\| \iota_r(\theta) - \iota_r(0) \right\|_{L^1}$, it now suffices to show that
    \[ \frac{\partial f}{\partial r} (r, \theta) > 0 \]
    for all $\theta \in (0, \pi]$ and $r \in \left( \sin \left( \frac{\theta}{2}\right), 1\right]$. Taking the derivative of (\ref{eq:iotar_L1}) gives us
    \[ \frac{\partial f}{\partial r} (r, \theta) = \frac{\pi}{2} - \arccos\left( \frac{\sin \left( \frac{\theta}{2} \right)}{r} \right) + \frac{1}{\sqrt{1 - \left(\frac{\sin \left( \frac{\theta}{2} \right)}{r} \right)^2}} \left( \left( \frac{\sin \left( \frac{\theta}{2} \right)}{r} \right)^3 - \left( \frac{\sin \left( \frac{\theta}{2} \right)}{r} \right) \right) \]
    for $\theta \in (0, \pi]$ and $r \in \left( \sin \left( \frac{\theta}{2}\right), 1\right]$.
    Making the change of variables $x = \frac{\sin \left( \frac{\theta}{2} \right)}{r}$, it remains to show that
    \[ g(x) \vc = \frac{\pi}{2} - \arccos(x) + \frac{x^3 - x}{\sqrt{1 - x^2}} = \frac{\pi}{2} - \arccos(x) - x \sqrt{1 - x^2} > 0\]
    for all $x \in (0, 1)$. This follows from observing that $g(0) = 0$ and
    \[ g'(x) = \frac{1}{\sqrt{1-x^2}} + \frac{x^2}{\sqrt{1-x^2}} - \sqrt{1-x^2} = \frac{2x^2}{\sqrt{1-x^2}} > 0 \]
    for all $x \in (0, 1).$
\end{proof}

\section{Auxiliary lemmas}
The following lemma is helpful for computations; we include a proof of it for completeness.

\begin{lemma}
\label{lemma:traceint}
    Let $m \in \N$ and $r > 0$. Let $\mathrm{area}(\S^{m-1})$ denote the surface area of $\S^{m-1}$ when embedded as a unit sphere in $\R^m$.
    For any $m \times m$ symmetric matrix $A \in \R^{m \times m}$ and integrable function $h \colon [0, r) \to \R$, 
    \[ \int_{B_{r}(0)} h(\|v\|) v^T Av dv =  \frac{\mathrm{area}(\S^{m-1}) \Tr(A)}{m} \int_0^{r}h(x) x^{m+1}dx. \]
\end{lemma}
\begin{proof}
Let $\mu_{\S^{m-1}}$ denote the uniform probability measure on $\S^{m-1}$ (embedded as a unit sphere in $\R^m$).
    By considering the spectral decomposition $A = U\Lambda U^T$, we obtain
    \begin{align*}
        \int_{B_{r}(0)} h(\|v\|) v^T Av dv &= \mathrm{area}(\S^{m-1})\int_0^{r} h(x) x^2 \left(\int_{\S^{m-1}}  \varphi^T A \varphi d\mu_{\S^{m-1}}(\varphi) \right) x^{m-1} dx \\
        &= \mathrm{area}(\S^{m-1}) \left(\int_{\S^{m-1}}  \varphi^T \Lambda \varphi d\mu_{\S^{m-1}}(\varphi) \right)  \left(\int_0^{r} h(x) x^{m+1} dx \right) \\
        &= \frac{\mathrm{area}(\S^{m-1}) \Tr(A)}{m} \int_0^{r}h(x) x^{m+1}dx.
    \end{align*}
\end{proof}

\begin{lemma}
    \label{lemma:betapos}
    Suppose $(\M, g)$ is a Riemannian manifold where $g$ is $C^2$, $(X, d)$ is a metric space and $\iota \colon \M \to X$ is a topological embedding. Then for any $x \in \M$ and $R > 0$,
    \[   \beta(x, R) \vc = \inf_{\substack{y \in \M \\ d_g(y, x) \geq R
    }} d(\iota(y), \iota(x))  \]
    is strictly positive.
\end{lemma}

\begin{proof}
Since $B_R(x) \vc = \{y \in \M \colon d_g(y, x) < R\}$ is open and $\iota$ is an embedding, $\iota(B_R(x))$ is open in $\iota(\M)$ with respect to the subspace topology from $(X, d)$. Suppose for contradiction that $\beta(x, R) = 0$. Then for all $s > 0$,
    \[ \widetilde{B}_s(\iota(x)) \vc = \{ z \in \iota(\M) \colon d(z, \iota(x)) < s \} \]
    contains $\iota(y)$ for some $y \notin B_R(x)$. Thus $\iota(x) \in \iota(B_R(x))$, but there does not exist $s > 0$ such that $\widetilde{B}_s(\iota(x)) \subset \iota(B_R(x))$. This implies that $\iota(B_R(x))$ is not open in $\iota(\M)$ with the subspace topology, which is a contradiction.
\end{proof}

\end{document}